\documentclass[preprint,12pt]{colt2023} % Include author names

\title[Information-Theoretic Neural Scaling Laws]{An Information-Theoretic Analysis of \\ Compute-Optimal Neural Scaling Laws}
\usepackage{times}
\usepackage{hyperref}
\usepackage{url}
\usepackage{graphicx}
\usepackage{xcolor}
\usepackage{enumerate}
\usepackage{textcomp}
\usepackage{mathrsfs}
\usepackage{bbm}
\usepackage{epigraph}
\usepackage{etoolbox}
\usepackage{thmtools, thm-restate}
\usepackage{proof-at-the-end}

\def\KL{\mathbf{d}_{\mathrm{KL}}}
\def\normal{\mathcal{N}}

\def\sign{{\rm sign}}
\def\proxytheta{\tilde{\theta}}
\def\relu{{\rm ReLU}}

\def\E{\mathbb{E}}

\def\H{\mathbb{H}}

\def\Z{\mathbb{Z}}
\def\I{\mathbb{I}}
\def\Pr{\mathbb{P}}

\def\1{\mathbf{1}}

\newcommand{\Lc}{\mathcal{L}}
\newcommand{\state}{U_t}
\newcommand{\finstate}{U_\infty}
\newcommand{\initstate}{U_0}
\newcommand{\finfunc}{F_{\finstate}}
\newcommand{\sphere}{\mathbb{S}^{d-1}}

\DeclareMathOperator*{\argmin}{arg\,min}
\usepackage{hyperref}

\usepackage[mathscr]{euscript}
\hypersetup{
    colorlinks=true,
    linkcolor=blue,
    filecolor=magenta,      
    urlcolor=purple,
}
\usepackage{xcolor}

\newcount\Comments
\newcommand{\kibitz}[2]{\ifnum\Comments=1{\textcolor{#1}{\textsf{\footnotesize #2}}}\fi}
\usepackage{color}
\definecolor{darkred}{rgb}{0.7,0,0}
\definecolor{darkgreen}{rgb}{0.0,0.5,0.0}
\definecolor{darkblue}{rgb}{0.0,0.0,0.5}
\definecolor{teal}{rgb}{0.0,0.5,0.5}

\coltauthor{%
 \Name{Hong Jun Jeon} \Email{hjjeon@stanford.edu}\\
 \Name{Benjamin {Van Roy}} \Email{bvr@stanford.edu}\\
 \addr Stanford University%
}

\begin{document}

\maketitle

\begin{abstract}%
  We study the compute-optimal trade-off between model and training data set sizes for large neural networks. Our result suggests a linear relation similar to that supported by the empirical analysis of \cite{chinchilla}. While that work studies transformer-based large language models trained on the MassiveText corpus \citep{gopher}, as a starting point for development of a mathematical theory, we focus on a simpler learning model and data generating process, each based on a neural network with a sigmoidal output unit and single hidden layer of ReLU activation units.  We introduce general error upper bounds for a class of algorithms which incrementally update a statistic (for example gradient descent).  For a particular learning model inspired by \cite{barron1993universal}, we establish an upper bound on the minimal information-theoretically achievable expected error as a function of model and data set sizes. We then derive allocations of computation that minimize this bound. We present empirical results which suggest that this approximation correctly identifies an asymptotic linear compute-optimal scaling. This approximation also generates new insights. Among other things, it suggests that, as the input dimension or latent space complexity grows, as might be the case for example if a longer history of tokens is taken as input to a language model, a larger fraction of the compute budget should be allocated to growing the learning model rather than training data.
\end{abstract}


\section{Introduction}

In recent years, state-of-the-art neural network models have grown immensely. GPT-3 \citep{gpt3}, for example, includes 175 billion learned parameters. While larger models have in general produced better results, they also require much more compute to train. It has become impractical to perform hyperparameter sweeps at the scale of these modern models. This is concerning since deep learning has previously heavily relied on tuning of hyperparameters via extensive trial and error.

When fitting a large language model, two important hyperparameters control $1)$ the size, measured in terms of the parameter count $p$, of the neural network model and $2)$ the number $t$ of training tokens. The compute budget $C$ scales with the product of these two quantities since each training token produces a gradient that is used to adjust each parameter of the neural network. For any compute budget $C$, one should carefully balance between $p$ and $t$. Too few training tokens leads to model estimation error, while too few parameters gives rise to mispecification error. As evaluating performance across multiple choices of $p$ and $t$ becomes computationally prohibitive at scale, alternative kinds of analysis are required to guide how to allocate computational resources.

Recently, \cite{nlm} and \cite{chinchilla} have proposed the following procedure for allocating a large compute budget: 
\begin{enumerate}
\item Evaluate test errors of models produced using various small compute budgets $C$ with many different allocations to parameters $n$ versus training tokens $t$.
\item Based on these results, estimate the optimal relation between $p$ and $t$, which minimizes error for each compute budget $C$.
\item Extrapolate to estimate the relation between $p$ and $t$ for large $C$.
\end{enumerate}
To give a sense of scales involved here, \cite{chinchilla} evaluate test errors across ``small'' models for which $p\times t$ ranges from around $10^{18}$ to $10^{22}$ and extrapolates out to ``large'' models at around $10^{24}$. \cite{nlm} and \cite{chinchilla} each extrapolate based on a hypothesized \emph{scaffolding function}.  \cite{nlm} guess a scaffolding function based on results observed in small scale experiments. \cite{chinchilla} carry out an informal and somewhat speculative mathematical analysis to guide their choice (see their Appendix D).

The analysis of \cite{chinchilla} is somewhat generic rather than specialized to the particular neural network architecture used in that paper. In this paper, we develop a more rigorous and formal mathematical analysis, bringing to bear information-theoretic tools that build on the framework and results of \cite{JeonNeurips2022,jeon_van_roy}. To keep things simple and concrete, we carry out the analysis with a particular data generating process for which neural networks are well-suited. The sorts of arguments developed by \cite{chinchilla} are just as relevant to this context as they are to language models.

\cite{chinchilla} suggest that the compute optimal trade-off between parameter count and number of training tokens is linear, though the authors expressed some doubt and considered other possibilities that are near-linear as well. We establish an upper bound on the minimal information-theoretically achievable expected error as a function of $p$ and $t$ and derive the relation required to minimize this bound for each compute budget. For large compute budgets, this relation is linear, as suggested by \cite{chinchilla}. Given that our data generating process differs from that considered by \cite{chinchilla}, we also carry out a computational study that corroborates this linear relation in our context.

Our main contributions include a first rigorous mathematical characterization of the compute-optimal efficient frontier for a neural network model and development of information-theoretic tools which enable that. A limitation of our analysis is in its simplified treatment of computational complexity as the product of the model and data set sizes; we do not assume any constraints on computation beyond those imposed by choices of $p$ and $t$. In particular, we analyze, algorithms which carry out perfect Bayesian inference, though with a model that is misspecificified due to its restricted size. While this abstracts away the details of practical training algorithms, empirical evidence suggests that our idealized framework leads to useful approximations \citep{YifanZhu2022}.  Additionally, while our general results hold for any gradient algorithm which 1) converges and 2) limiting solution when trained on an infinite stream of data is deterministic given the initialization, we are unable determine the distribution of this limiting solution.  Instead, we posit a potential limiting distribution which exhibits interesting novel theoretical results and facilitate the computation of a compute-optimal trade-off which is corroborated with empirical evidence.  In spite of these limitations, we hope our results set the stage for further mathematical work to guide hyperparameter selection when training large neural networks. 
Furthermore, our approach can give rise to new insights not captured by previous analyses. For example, we consider implications of our bound on how computation ought to be apportioned as the input space dimension or the latent space complexity of the data generating process grows, as may be the case, for example, when taking a longer history of tokens as input to a language model. We next summarize our formulation and insights.

\section{Formulation and Insights}

Our analysis is inspired by the empirical study of \cite{chinchilla}, which uses a simple formula to approximate flops of computation: $C = p \times t$. In other words, the number $C$ of flops is taken to be the product of the number $p$ of model parameters and the number $t$ of training data samples. We will similarly use this formula, which is motivated by the fact that, when training a large language model, each training data sample is processed only once and the computation required by each scales roughly with the number of model parameters. Sample are not processed more than once because there is insufficient computation time to work through vast text data corpi.

\cite{chinchilla} train transformer models on the MassiveText corpus \citep{gopher}. However, in order to simplify our analysis, we restrict attention to a particular synthetic data generating process and learning model, each based on a neural network architecture.
Our data generating process samples inputs $X_t$ iid from a standard multivariate Gaussian and produces binary labels $Y_{t+1}$ via a neural network with a sigmoidal output unit and a single infinitely wide hidden layer of ReLu activation units. We denote by $F$ the associated mapping from input $X_t$ to the sigmoid output. A prior distribution over output weights characterizes initial uncertainty about $F$. This distribution depends on a scalar hyperparameter $K$, which controls the expected complexity of $F$. In particular, output weights are sampled such that the number of hidden units required to attain any particular level of loss scales with $K$.

As a learning model, we consider a similar neural network architecture, again with a linear output node and single hidden layer of ReLu activation units. However, while the data generating process architecture includes an infinite number of hidden units, we limit the number in our learning model to $C/t$. We aim to learn an approximation $\tilde{F}$ to $F$ within the constrained learning model class. Note that $\tilde{F}$ depends on $F$, and therefore, uncertainty about $F$ induces uncertainty about $\tilde{F}$. It is this uncertainty about $\tilde{F}$ that is reduced as the learning model is trained.

\begin{figure}[!ht]
\centering
\includegraphics[scale=0.6]{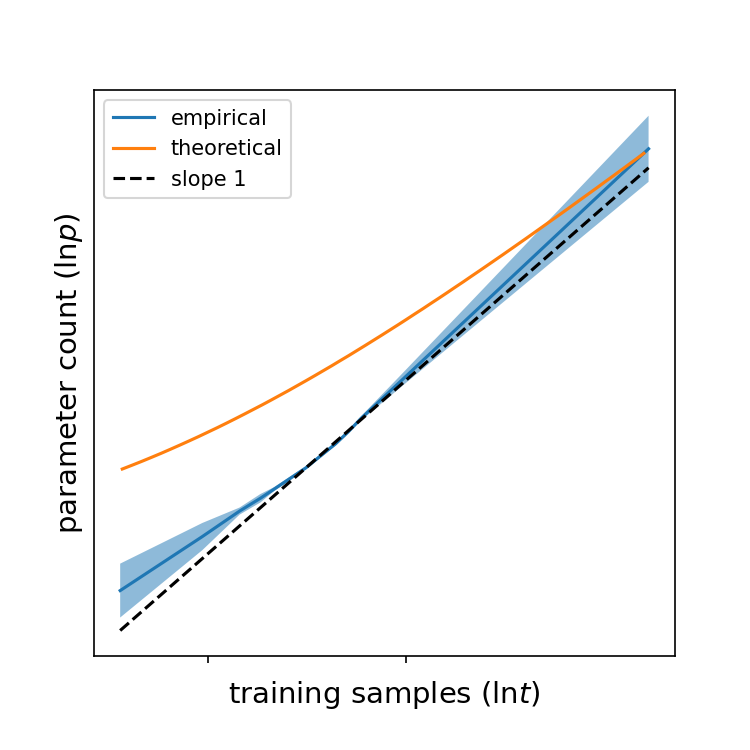}
\caption{Estimates of compute-optimal allocations between parameter count and training samples.}
\label{fig:efficient-frontier}
\end{figure}

In order to model uncertainty, that we treat the data generating process $F$, the learning model $\tilde{F}$, and observed data as random variables. These and all other random variables we introduce are defined with respect to a common probability space $(\Omega, \mathbb{F}, \mathbb{P})$.

We consider learning algorithms which maintain a statistic $U_t$ which obeys an incremental update procedure $\psi$:
$$U_{t+1} \sim\psi(\cdot|U_t, X_t, Y_{t+1}).$$
Note that if $U_t$ are parameters of a neural network and
$$U_{t+1} = U_t - \alpha \nabla \Lc(U_t, X_t, Y_{t+1}),$$
where $\Lc$ is a loss function and $\alpha$ a known step-size, then we recover gradient descent. For \emph{inference}, we consider an idealized algorithm which performs exact Bayesian inference w.r.t $U_t$. We refer to this prediction as $P_t$ and concretely, $P_t = \Pr(Y_{t+1}|U_t, X_t)$. While this abstracts away the details of neural network inference which results from optimization via stochastic gradient descent, empirical evidence suggests that they can approximately attain the performance of our idealized counterpart \citep{YifanZhu2022}. To assess performance, we characterize out-of-sample cross-entropy loss of the prediction $\Pr(Y_{t+1}=1 | U_t, X_t)$, which is conditioned on $U_t$ and $X_t$. Our results inform how a compute budget $C$ should be apportioned to a number $p$ of model parameters versus a number $t$ of training samples in order to minimize out-of-sample cross-entropy loss.

Figure \ref{fig:efficient-frontier} plots estimates of compute-optimal allocations between model size and training samples across different compute budgets. Due to limitations in our computational resources, we ran experiments with only up to around one hundred and fifty thousand training samples. The ticks on horizontal axis indicate the range spanned by our experiments. Each point on the blue curve is produced by a linear model fit to a subset of our experimental results; further detail is provided in Section \ref{se:empirical-results}. If the compute-optimal model size grows linearly in the compute-optimal number of training samples, that would be reflected through unit slope. The dotted black line has unit slope and is plotted for reference. The blue curve represents empirical estimates produced using synthetic data and stochastic gradient descent, as explained in further detail in Section \ref{se:efficient-frontier}. The shaded blue region represents a standard error confidence interval. The orange curve is our theoretical approximation. As the compute budget grows, the slopes of both curves become statistically indistinguishable from one. Furthermore, the curves appear to merge. Overall, these results are consistent with empirical observations of \cite{chinchilla} and suggest that our theory captures the right qualitative behavior as the compute budget grows.

\begin{figure}[!ht]
\centering
\includegraphics[scale=0.40]{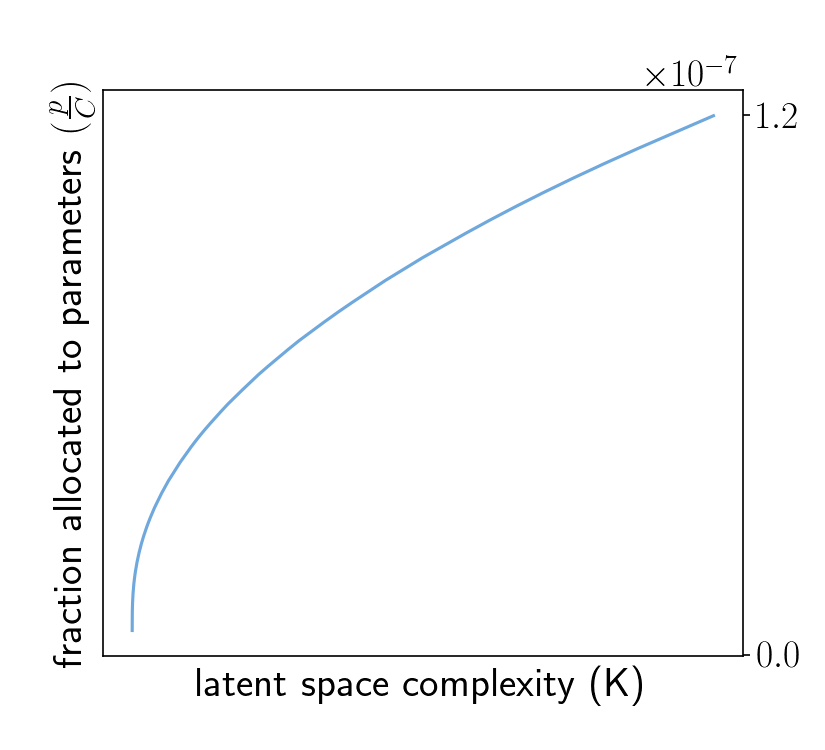}
\includegraphics[scale=0.40]{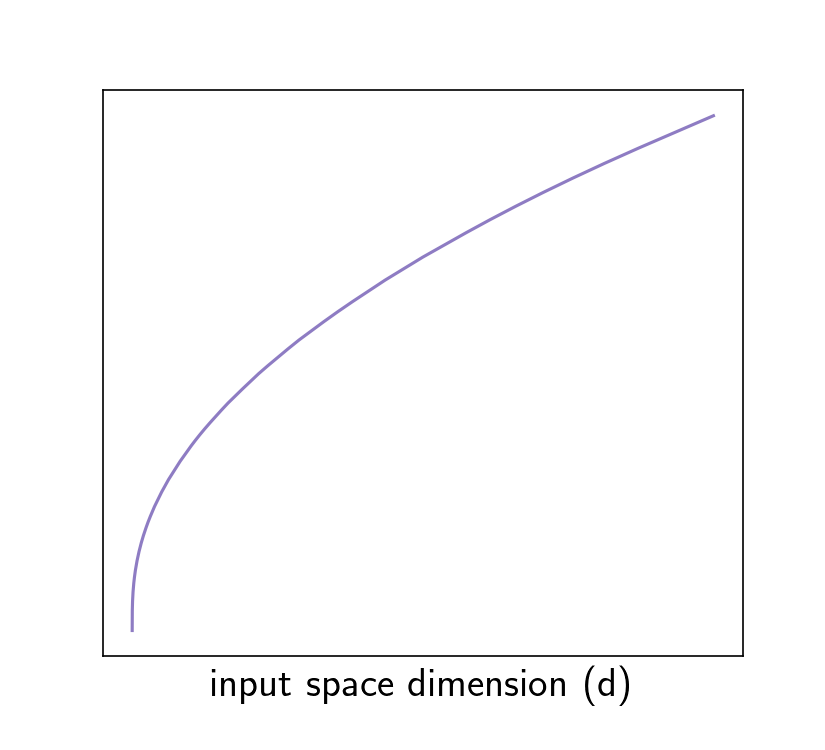}
\caption{The parameter count as a fraction of compute budget increases with complexity of the data generating process. Note that the vertical axis is identical across both plots. Interestingly, parameter allocations are \emph{identical} whether we vary $d$ or $K$.
The shape of the curves seems to be independent of $C$. These plots were generated with $C = 2^{50}$.}
\label{fig:data-complexity}
\end{figure}

Figure \ref{fig:data-complexity} is based on theoretical approximation. Assuming that the approximation is accurate, it plots the optimal fraction of the compute budget allocated to the parameter count $p$ as a function of the complexity of the data generating process which is determined by $K$ and $d$. The trend suggests that, as complexity grows, an increasing fraction of computation should be allocated to scaling models as opposed to increasing the number of training data samples. Furthermore, this relationship is concave. With large language models, it is natural to expect that complexity of the data will grow when a longer history of tokens is taken as input. As such, this figure offers insight into how the optimal allocation of computation ought to change as models are designed to ingest longer histories.

\section{Data Generating Process}
\label{sec:models_error}

We now present the data generating process and learning model we study. We establish an error bound that captures dependence on the learning model size $p$ and the number of training samples $t$.

\subsection{Data generating process}

Our data generating process is represented by a neural network with $d$ inputs, a single asymptotically wide hidden layer of ReLU activation units, and a sigmoidal output unit. We denote by $F$ the associated mapping from input to sigmoidal output. In particular, inputs and binary labels are generated according to $X_t\overset{iid}{\sim}\normal(0, I_d)$ and $\Pr(Y_{t+1} = 1 | \theta, X_t) = F(X_t)$. 

\begin{figure}
    \centering
    \includegraphics[scale=0.25]{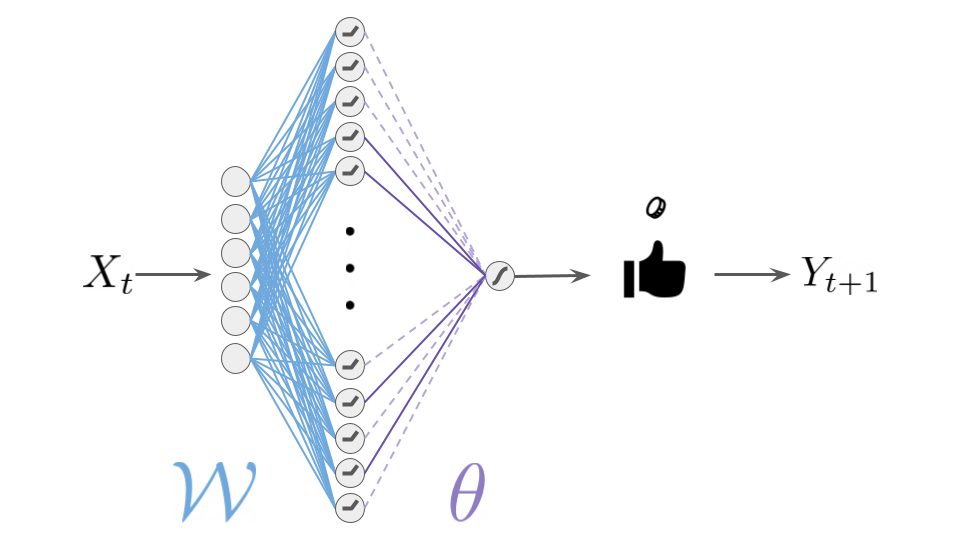}
    \caption{This diagram depicts our data generating process. Each basis function is parameterized by some $w \in \mathcal{W}$. After passing through a ReLU, the output is multiplied by $\theta$ sampled from the Dirichlet process outlined in the text. Despite the fact that $\theta$ has dimension $\infty$, the scale parameter $K$ ensures that only $\approx K$ dimensions have significant magnitude. A sigmoidal function is applied to the output to designate the probability that $Y_{t+1}$ is $1$.}
    \label{fig:data_generating_process}
\end{figure}

Let $S^{d-1}$ denote the $d$-dimensional unit sphere. We index hidden units by a vector $w \in S^{d-1}$. The corresponding output is given by $\relu(w^\top X_t)$. We sample hidden unit weights $\theta$, also indexed by $w \in S^{d-1}$, in a manner such that, with probability one, $\theta_w$ is nonzero for only a countable set of indices. Let $\mathcal{W} = \{w \in S^{d-1} : |\theta_w| > 0\}$ denote this set. Then, 
$$F(X_t) = \frac{1}{1+\exp\left(- \sqrt{K+1}\sum_{w \in \mathcal{W}} \theta_w \relu(w^\top X_t)\right)}.$$

We now described how $\theta$ is generated. First, $\overline{\theta}$ is sampled from a Dirichlet process with base distribution $\mathrm{uniform}(S^{d-1})$ and scale parameter $K$. That this almost surely results in countable support $\mathcal{W}$ is a well-known property of Dirichlet processes. For each $w \in \mathcal{W}$, we randomly flip the sign of the weight according to
$$\theta_w = \left\{\begin{array}{ll}
\overline{\theta}_w \qquad & \text{with probability } 1/2, \\
-\overline{\theta}_w \qquad & \text{otherwise.}
\end{array}\right.$$
Note that the distribution of $\theta$ depends on the scalar hyperparameter $K$, which controls the expected complexity of $F$. In particular, the number of hidden units required to attain any particular level of loss scales with $K$.

A few observations may help interpret certain choices in our data generating process.  Let us denote the input to the sigmoidal output unit by $f(x) = \sqrt{K+1} \sum_{w \in \mathcal{W}} \theta_w \relu(w^\top x)$. Its mean and covariance structure satisfy
$$\E[f(x)] = 0, \qquad \E[f(x)^2] = \frac{1}{2},\qquad \text{and} \qquad \E[f(x) f(x')] = \frac{1}{2\pi},$$
where $x, x'$ are independently sampled from $\normal(0,I_d)$. This model remains nontrivial as $d$ and $K$ grow as all of the above quantities are invariant in $d$ and $K$.

\subsection{Motivating the Data Generating Process}
In our simplified theoretical analysis of neural scaling laws, we arrived at the above data generating process after considering two necessary conditions: $1)$ The data generating process is infinitely complex $2)$ Meaningful progress can be made with finite compute constraints.  Condition $1)$ is necessary because we empirically observe that as we increase model size and dataset size, out-of-sample performance continues to improve.  If the data generating process were finitely complex, there would exist a threshold after which further increasing the model size would theoretically provide no additional benefit.  Our data generating process expresses ''infinite complexity`` from the fact that it is characterized by a neural network of \emph{infinite} width.  However, one could easily devise data generating processes of infinite complexity for which a learning model with \emph{finite} compute would have no hope of achieving meaningful performance.  Again, we empirically observe that even if the complexity necessary for nature to generate human speech is effectively infinite, transformer models trained with finite compute constraints still provide meaningful performance.  As such, we deem condition $2)$ necessary.  While we defer verification that this condition holds for our data generating process to the appendix, the finite scale parameter $K$ is responsible for facilitating this condition.  

\section{General Theoretical Results}

In this section we will provide general theoretical results which extend to all supervised learning problems and algorithms which abide by several natural assumptions we will introduce shortly.  The purpose of this section is to provide an intuitive yet rigorous analysis of the factors at play when considering the optimal allocation of compute.

\subsection{Algorithmic Assumptions}
We restrict our attention to algorithms which perform \emph{incremental} updates w.r.t a statistic. For all $t\in \mathbb{Z}_{+}$, we let $\state$ denote the statistic at time $t$. In deep learning, this statistic is often the neural network parameters. In this sense, $\initstate$ denotes the network weights at initialization and subsequent parameter updates are made via gradient descent:
$$U_{t+1}\ =\ \state - \alpha_t \nabla\Lc(Y_{t+1}, X_t, \state),$$
where $\alpha_t$ denotes the step-size at update $t$ and $\Lc$ a loss function.  For network parameters $\state$, we use $F_{\state}$ to denote the \emph{function} that results from running a forward pass on the network with weights $\state$. In order to demonstrate concrete results for such gradient algorithms, we make the following assumptions:

\begin{enumerate}
    \item \textit{(convergence)} We assume that our algorithm is convergent i.e, with probability $1$ over the randomness in the data $H_t$ and initialization $\initstate$, the following limit exists:
    $$\lim_{t\to\infty} F_{\state}\ \overset{a.s.}{=}\ \finfunc,$$
    where $\finfunc$ denotes the limit.\\
    \item \textit{($\finfunc$ is not influenced by randomness in data)}\\ Let $(\tilde{X}_0,\tilde{Y}_{1}, \tilde{X}_1, \tilde{Y}_2, \ldots)$ be a random process which consists of data pairs which are iid conditioned on $F$ and for all $t\in \Z_{+}$, let $\tilde{H}_t = (\tilde{X}_0, \tilde{Y}_1,\ldots, \tilde{X}_{t-1}, \tilde{Y}_t, \tilde{X}_t)$ and $\tilde{U}_t$ denote the model parameters after running our algorithm from the same initialization $\initstate$ but with data $\initstate$. Then,
    $$\lim_{t\to\infty} F_{\state}\ \overset{a.s.}{=}\ \lim_{t\to\infty} F_{\tilde{U}_t}\ \overset{a.s.}{=}\ \finfunc.$$
\end{enumerate}

Assumption $(1)$ simply states that our step-sizes are set so that the algorithm converges almost surely. It is unclear whether this condition holds for the training procedure of \cite{chinchilla} but in the context of a theoretical analysis, this assumption is mild.

Assumption $(2)$ is more speculative but it is a condition that we would \emph{expect} to hold for an effective learning algorithm. In the limit of infinite data, we should not expect the algorithm's solution $\finfunc$ to depend on the realization of the data. However, since gradient descent on neural networks with non-linearities is a \emph{non-convex} optimization problem, we allow for the solution to depend on the initialization $\initstate$.  Going forward, we will refer to $F_{U_\infty}$ as the \emph{learning model} as it represents the full extent of what is learnable via our algorithm.

Under these two assumptions, we derive general results which decompose the out-of-sample error into $4$ terms: irreducible, misspecification, estimation, and inferential errors.

\subsection{Error}

We evaluate performance via cross-entropy error. Given an input $X_t$ and a prediction $P_t$, which takes the form of a probability in $[0,1]$, of the label $Y_{t+1}$, the cross-entropy error is defined by
$$L_t(P_t) = - Y_{t+1} \ln P_t - (1-Y_{t+1}) \ln (1-P_t).$$
Prediction $P_t$ is allowed to depend on the input $X_t$ and the model parameters $\state$. In deep learning, it is common for $P_t = F_{\state}(X_t)$. We use $P^*_t$ to denote $\Pr(Y_{t+1}|F, X_t)$ or the \emph{omniscient} predictor as $P^*_t$ produces predictions for $Y_{t+1}$ conditioned on $F$. We also introduce the predictor $\hat{P}_t = \Pr(Y_{t+1}|\state, X_t)$ the posterior-predictive distribution of $Y_{t+1}$ conditioned on input $X_t$ and $\state$. It is no surprise that $\hat{P}_t$ \emph{minimizes} $L_t(P_t)$ for all predictors which are allowed to depend on $\state$ and $X_t$. The following result decomposes $\E[L_t(P_t)]$ into $4$ intuitive sources of error for any algorithm which satisfies assumptions $(1)$ and $(2)$ stated in the prior section.

\begin{restatable}{theorem}{lossDecomp}
    For all $t\in \mathbb{Z}_+$, if $\state$ is produced by an algorithm which satisfies assumptions $(1)$ and $(2)$, then
    $$\E[L_t(P_t)] = \underbrace{\E[-\ln P^*_t]}_{\rm bayes\ error} + \underbrace{\I(Y_{t+1};F|\finfunc, \state, X_t)}_{\rm misspecification\ error} + \underbrace{\I(Y_{t+1};\finfunc|\state, X_t)}_{\rm estimation\ error} + \underbrace{\E[\KL(\hat{P}_t\|P_t)]}_{\rm inferential\ error}.$$
\end{restatable}

The \emph{bayes error} represents error experienced by even the omniscient agent. Since cross-entropy loss is not baselined at $0$, the bayes error contributes to a form of \emph{irreducible error}. On the other end, we have \emph{inferential error} which describes the shortfall of producing predictions of $Y_{t+1}$ via $P_t$ as opposed to $\hat{P}_{t}$. In practical problem instances, $\hat{P}_t$ cannot be computed because the structure of the data generating process is not known. As a result, it is common to simply use $P_t = F_{\state}(X_t)$ as mentioned above.  We currently do not have adequate theoretical tools to study the inferential error, so analysis beyond this point will pertain to $L_t(\hat{P}_t)$ for which the inferential error is $0$.

The \emph{misspecification error} represents error which remains even as the number of samples grows to $\infty$. The error can instead be written as the following expected KL divergence:
$$\I(Y_{t+1};F|\finfunc, \state, X_t) = \E\left[\KL(P^*_t\|\Pr(Y_{t+1}\in\cdot|\finfunc, \state, X_t))\right].$$
Since the predictor $\Pr(Y_{t+1}\in\cdot|\finfunc, \state, X_t)$ conditions on $\finfunc$, the learned model which is the result of training on infinite data, this error is a result of the learning model being too simplistic in comparison to the data generating process. Since our data generating process is identified by an \emph{infinite width} neural network, the misspecification error will be nonzero for any learning model of finite width. However, we expect this quantity to \emph{decrease} as we \emph{increase} the parameter count of the learning model. While additional assumptions must be made in order to eliminate the conditioning on $\state$, we can always \emph{upper bound} the misspecification error by $\I(Y_{t+1};F|\finfunc, \initstate, X_t)$ (Lemma \ref{le:miss_error_bound} of Appendix \ref{apdx:general}) to make the quantity explicitly independent of the dataset size.
% \hong{this can actually be made into equality if we assume a monotonicity condition on the training error before/after a point has been trained on} 

Finally, the \emph{estimation error} represents statistical error which occurs as a result of estimating the learning model under finite samples. The error can be instead written as the following expected KL divergence:
$$\I(Y_{t+1};\finfunc|\state, X_t) = \E\left[\KL\left(\Pr(Y_{t+1}\in\cdot|\finfunc, \state, X_t) \| \hat{P}_{t}\right)\right].$$
The KL divergence measures the difference between predictions made with knowledge of $\finfunc$ and just $\state$, the model parameters after updated with the information of $t$ data points. As classical statistics would dictate, estimation error will depend on both the parameter count and the data set size.

\subsection{Error bounds}

We now make an additional third assumption on the learning algorithm which facilitates analysis.
\begin{enumerate}
  \setcounter{enumi}{2}
  \item \textit{(reducible error is monotonically decreasing in dataset size)} For all $t \in \mathbb{Z}_{+}$,
  $$\E\left[\KL\left(P^*_{t+1}\|\hat{P}_{t+1}\right)\right] \leq \E\left[\KL\left(P^*_{t}\|\hat{P}_t\right)\right]$$
\end{enumerate}
This assumption states that the expected error of $\hat{P}_{t+1}$ is less than that of $\hat{P}_{t}$. Note that this assumption only needs to hold in \emph{expectation}. Much like Assumption $(2)$, while this is hard to demonstrate for practical algorithms, it is again a condition that we would \emph{expect} to hold for an effective learning algorithm. Given more data, the algorithm should obtain smaller error in expectation.  We now present two theoretical results: $(1)$ a general result which upper bounds the reducible error for an algorithm which satisfies Assumptions $1-3$ and $(2)$ the result applied to our data generating process

\begin{theorem}{\bf(reducible error upper bound)}
    For all $t\in \mathbb{Z}_+$, if $\state$ is produced by an algorithm which satisfies assumptions $(1)$, $(2)$, and $(3)$ then,
    $$\E[\KL(P^*_t\|\hat{P}_t)] \leq \I(Y_{t+1};F|\finfunc, \initstate, X_t) + \frac{\H(\finfunc)}{t}.$$
\end{theorem}

In practice it is difficult to determine what random variable $\finfunc$ is for a finite width relu neural network trained via sgd.  As a result, we posit a particular random variable $\finfunc$ which abides by the architectural constraints and seemingly provides meaningful insight when empirically validated.

We now present an upper bound on the cross-entropy error of an optimal neural network learning model of width $n$ and $t$ training samples.

\section{Our Learning Model}
As aforementioned, while ideally we would like to characterize the learning model $\finfunc$ as a random variable which represents the result of running SGD on a stream of iid data pairs from random initialization, this work cannot provide such a result.  Instead, we provide a hypothetical learning model which produces both interesting novel error bounds and recovers an optimal compute allocation which mirrors that of \citep{chinchilla}.  In the following section, we further corroborate these results by empirically estimating the compute optimal allocation with neural networks trained via SGD.

While we choose a particular learning model, it still abides by the relevant architectural constraints pertinent to optimal compute allocation.  As such, our learning model $\finfunc$ is a neural network with input dimension $d$ and width $n$.  Each hidden unit is parameterized by weights $w_i \in S^{d-1}$.  Recall that realizations of $\overline{\theta}$ are discrete distributions over $S^{d-1}$ sampled from a Dirichlet process to produce output layer weights $\theta$ for our data generating process.  We consider $\finfunc$ for which 
$$w_i = \begin{cases}
    w & \text{w.p. } \overline{\theta}_w\\
\end{cases}.$$
The weight assigned to the output of the $i$th hidden unit is taken to be $\sign(\theta_{w_i}) \sqrt{K+1}/n$, so that
$$\finfunc(X_t) = \frac{1}{1+\exp\left(- \frac{\sqrt{K+1}}{n} \sum_{i=1}^n \sign(\theta_{w_i}) \relu(\left(\bar{w}_{i,\epsilon}\right)^\top X_t)\right)},$$
where $\bar{w}_{i,\delta}$ is a quantization of $w_i$ for which $\|w_i - \bar{w}_{i,\delta}\|_2 \leq \delta$ for all $i$.  

This function is constructed by sampling $n$ hidden units from the data generating process with replacement, quantizing the input weights of each of these hidden units up to fidelity $\delta$, and averaging their outputs.  As such, misspecification error stems from two sources of imprecision: $1)$ the input weight quantization and $2)$ randomly sampling and averaging $n$ hidden units to approximate the linear output layer.  Approximations of this sort date back to the work of \cite{barron1993universal}, which established that finite width neural networks can provide accurate approximations of functions from a nonparametric class.  Note that $\delta$, the quantization tolerance, is a free parameter that we can optimize for any pair $(n, t)$ to produce a tight bound. This perturbation is introduced to keep the entropy of the learning model finite.
 
In the remainder of this section, we present upper bounds on the misspecification error and estimation error of the algorithm which produces the prediction $\hat{P}_t = \Pr(Y_{t+1}|\finfunc, \initstate, X_t)$ for all $t$.

\subsection{Error Bounds for Our Learning Model}
We begin by providing an upper bound on the \emph{misspecification} error of our learning model.

\begin{restatable}{theorem}{misspecificationUb}{\bf (misspecification error upper bound)}
    \label{th:mis_error_ub}
    For all $n, t, d, K \in \Z_{++}$ and $\delta > 0$,
    %$n\cdot\proxytheta \sim {\rm Multi}(n, \theta)$,  and $\tilde{\delta} = \left(\delta_1\cdot\mathbbm{1}_{[\theta_1 > 0]}, \ldots, \delta_M\cdot\mathbbm{1}_{[\theta_M > 0]}\right)$, then
    $$\I(Y_{t+1};F|\finfunc, \state, X_t) \leq \frac{3(K+1)}{n} + 2d\delta^2.$$
\end{restatable}
This result matches the error bound derived in \citep{barron1993universal} for $\delta^2 = 1/dn$. The first term captures error incurred from our learning model's finite width $n$. The second term captures error from fitting with hidden units that are perturbed by variance $\delta^2$ noise. The error increases as the complexity of the environment $d, K$ increases, and decreases as the number of parameters in our learning model $n$ increases and the magnitude of perturbations $\delta^2$ to our basis function decrease.

We next derive an upper bound for the \emph{estimation} error incurred by our learning model.

\begin{restatable}{theorem}{estErrorUb}{\bf{(estimation error upper bound)}}
    \label{th:est_error_ub}
    For all $n, t, d, K \in \Z_{++}$ and $\delta > 0$, if $K \geq 2$, then
    $$\I(Y_{t+1};\finfunc|\state, X_t) \leq \frac{K\ln\left(1 + \frac{n}{K}\right)\cdot\left(\ln\left(2n\right) + d\ln\left(\frac{3}{\delta}\right)\right)}{t}.$$
\end{restatable}

Just as in \citep{barron1994approximation}, the estimation error decays at a rate of $\mathcal{O}\left(1/t\right)$. However, the interesting distinction is that Theorem \ref{th:est_error_ub} suggests that while the estimation error increases in $n$, the dependence is only \emph{logarithmic} as opposed to \emph{linear}. By combining Theorems \ref{th:mis_error_ub} and \ref{th:est_error_ub} and minimizing with respect to $\delta^2$, we have the following result.

\begin{restatable}{corollary}{ceUB}{\bf(cross-entropy error upper bound)}
    \label{cor:ce_ub}
    For all $n, t, d, K\in \Z_{++}$, and $\delta > 0$, if $K \geq 2$, then
    $$\E\left[L_t\left(\hat{P}_t\right)\right] \leq \underbrace{L^*}_{\rm bayes\ error} + \underbrace{\frac{3(K+1)}{n} + 2d\delta^2}_{\rm misspecification\ error} + \underbrace{\frac{K\ln\left(1 + \frac{n}{K}\right)\cdot\left(\ln(2n) + d\ln\left(\frac{3}{\delta}\right)\right)}{t}}_{\rm estimation\ error}.$$
\end{restatable}

The second term represent the misspecification error incurred from using finite width $n$ to approximate a data generating process of countably infinite width. Intuitively, this error increases with the environment complexity $K$ and input dimension $d$ and decreases with the learning model complexity $n$ and increased quantization precision $\delta$. The third term represents the approximation error incurred from estimating output layer weights from $t$ training samples. This quantity intuitively decreases as the sample size $t$ increases. Meanwhile, it increases with both increased complexity of the data generating process $K$ and the learning model $(n, \delta)$. Notably, the estimation error bound is only \emph{logarithmic} in $n$.  We further optimize this result by selecting the $\delta$ which minimizes the upper bound:

\begin{restatable}{theorem}{optCeUb}{\bf{cross-entropy error upper bound}}
    \label{th:ce_ub_final}
    For all $n, t, d, K\in \Z_{++}$, if $K \geq 2$, then
    $$\E\left[L_t\left(\hat{P}_t\right)\right] \leq L^*+\frac{3(K+1)}{n} + \frac{dK\ln\left(1+\frac{n}{K}\right)\left(1+\ln(36t)+\frac{2}{d}\ln(2n)\right)}{2t}.$$
\end{restatable}

In the section to come, we will derive a compute optimal allocation by choosing $n, t$ which minimizes the above upper bound under the constraint that $n \times t = C$.  However, we first compare this theoretical result to an existing result of \citep{barron1994approximation}.

\subsection{Relation to Barron's Bounds}

It is well known that a learning model of this form can achieve misspecification error that decays at rate $\mathcal{O}(K/n)$ \citep{barron1993universal} when measured in \emph{squared error}. \cite{barron1993universal} derived a result which bounded the misspecification error incurred by approximating any function with a particular Fourier representation via a single hidden layer neural network with $n$ hidden units. Their followup analysis in \citep{barron1994approximation} studies both the misspecification error and the \emph{estimation error}. They upper bound the estimation error by analyzing the number of samples required by a width $n$ neural network that minimizes a regularized least-squares objective. Notably, the techniques used to bound the misspecification error are almost entirely divorced from those used to bound the estimation error. If these bounds were tight, one could derive a scaling law for which the optimal number of training samples is $\tilde{\mathcal{O}}(C^{2/3})$ while the optimal parameter count is $\tilde{\mathcal{O}}(C^{1/3})$, which unfortunately does not match empirical results highlighted in Figure \ref{fig:efficient-frontier}. 

\cite{E_2019} also later derived a bound in the same setting. Their estimation error bound is actually independent of $n$, the width of the network, but pays the price in its dependence on $t$, which is $\tilde{\mathcal{O}}(1/\sqrt{t})$. This results in the same incorrect scaling law as \citep{barron1994approximation}, one in which the optimal number of training samples is $\tilde{O}(C^{2/3})$ and optimal parameter count is $\mathcal{O}(C^{1/3})$.

We hypothesize that our analysis is able to recover the appropriate scaling law because our framework allows us to coherently reason about the misspecification and estimation error of a learning model in tandem as opposed to independently. The scaling law of \cite{barron1994approximation} indicates that the number of training samples grows \emph{faster} than what is optimal for a neural network learning via stochastic gradient descent. This suggests that either the estimation error bounds of \cite{barron1994approximation} are loose, or neural networks trained via stochastic gradient descent do not require as many samples to avoid overfitting as their regularized least-squares solution does. Inspired by the observation that stochastic gradient descent seems to match the performance of an optimal Bayesian learner \citep{YifanZhu2022}, we attempted to understand the performance of stochastic gradient descent applied to a single-hidden layer neural network of width $n$ by instead analyzing the performance of an \emph{optimal Bayesian learner} that tries to learn $\tilde{F}$ from observed training samples.

\subsection{Compute-Optimal Allocations}
\label{se:efficient-frontier}
In this section, we provide a bound on the width of the learning model that minimizes the upper bound presented in Theorem \ref{th:ce_ub_final} subject to compute constraints. Recall that $C \in \mathbb{Z}_{++}$ denotes our \emph{compute budget} which is the product of the \emph{parameter count} and \emph{number of training samples}.
\subsection{Theoretical Approximation}
The following result provides an upper bound on the optimal width of the learning model that minimizes the error upper bound from Theorem \ref{th:ce_ub_final} subject to compute constraint $n\cdot d\cdot t \leq C$. Note that in reality the parameter count is $n\cdot(d+1)$ but we use $n\cdot d$ as it reduces clutter without meaningfully impacting the result.

\begin{restatable}{theorem}{effFront}{\bf(compute-optimal width approximation)}
\label{th:eff_frontier}
For all $d, K, C \in \mathbb{Z}_{++}$, if $K\geq 2, d \geq 3$, and
\begin{align*}
n^*
& = \argmin_{n \leq C/t d} \left(\frac{3(K+1)}{n}+\frac{K\ln\left(1+\frac{n}{K}\right)\cdot\ln\left(2n\right)}{t} + \frac{dK\ln\left(1+\frac{n}{K}\right)\left(\frac{1}{2} + \frac{1}{2}\ln\left(36t\right)\right)}{t} \right),
\end{align*}
then
$$n^* \leq \min\left\{\frac{\sqrt{\frac{6C}{d^2\ln(\ln(d))}}}{W_0\left(\sqrt{\frac{6C}{d^2K^2\ln(\ln(d))}}\right)}, \sqrt{\frac{6C}{d^2\ln(\ln(d))}} + \sqrt{\frac{6C}{d^2\ln(\ln(d))} + 4\sqrt{\frac{6CK^2}{d^2\ln(\ln(d))}}}\right\},$$
where $W_0$ is the Lambert $W$ function on branch $0$.
\end{restatable}

While Theorem \ref{th:eff_frontier} is stated only as an upper bound due to cumbersome terms that preclude a closed-form solution, we believe that it is a close approximation that captures the appropriate scaling in $d, K,$ and $C$. Note that this is not necessarily an upper bound on the compute-optimal width, but rather an upper bound on the width that minimizes the error upper bound of Theorem \ref{th:ce_ub_final}.

Recall that the parameter count is $n\cdot d$ so ignoring $\ln\ln$ factors, Theorem \ref{th:eff_frontier} suggests that the optimal parameter count for minimizing our error upper bound scales in the following way:
$$n^*\cdot d \leq \min\left\{\frac{\sqrt{6C}}{W_0\left(\sqrt{\frac{6C}{d^2K^2}}\right)}, \sqrt{6C} + \sqrt{6C + 4\sqrt{6Cd^2K^2}}\right\}$$
Note that $W_0(x) \approx \ln(x) - \ln(\ln(x))$ so we can establish that both terms in the min are $\tilde{O}\left(\sqrt{C}\right)$. Another interesting property is that in of these quantities the role of $d$ and $K$ is equivalent (up to the $\ln\ln$ factors we ignored). Therefore, irrespective of which of $d$ or $K$ are perturbed, the implications on the optimal parameter count upper bound is equivalent. This is not just a result of Theorem \ref{th:eff_frontier} being loose as experimentally we verified that the value of $n\cdot d$ which minimizes the upper bound in Theorem \ref{th:ce_ub_final} truly behaves identically for perturbations in $d$ and $K$ as apparent in Figure \ref{fig:data-complexity}.

\begin{remark}{\bf (simplified optimal compute allocation)}
    If we further analyze the bound, we see that for extrapolation i.e. large $C$, the first term will be the lesser of the two quantities. As a result, we will focus attention on this quantity. Since $W_0(x) \sim \ln(x) - \ln(\ln(x))$, the optimal parameter count will be:
    $$n^*\cdot d = \mathcal{O}\left(\frac{\sqrt{C}}{\ln\sqrt{\frac{C}{d^2K^2}}}\right).$$
    The optimal dataset size would therefore be 
    $$\mathcal{O}\left(\sqrt{C}\ln\sqrt{\frac{C}{d^2K^2}}\right).$$
\end{remark}
This is consistent with the compute-optimal scaling laws of \citep{chinchilla} up to logarithmic factors in $C$. However, our analysis additionally provides insight into how these scaling laws ought to vary with changes in the complexity $(d, K)$ of the data generating process.

\section{Empirical Estimates}
\label{se:empirical-results}

While our theoretical approximation is consistent with the linear scaling observed in \citep{chinchilla}, we study a much simpler data generating process. To corroborate our theoretical approximation, we performed a set of experiments under our data generating process and compared the empirical compute-optimal trade-off with our theoretical approximation. While we leave the full details to Appendix \ref{apdx:experiment}, we will provide a high-level overview of the experiments in this section.

To generate Figure \ref{fig:efficient-frontier}, we first specified an increasing sequence of compute budgets. For each compute budget $C$, we performed a line search over various pairs of network width and number of training samples $(n, t)$ which satisfied the constraint $nt(d+1) = C$. For each configuration $(n, t)$, we conducted several runs of Adam optimizer for a fixed number of epochs. We performed and averaged this procedure over many realizations of $F$ and $H_t$, both sampled according to our data generating process. Note that each run was influenced by both the randomness in the optimization and in the data.

As in \citep{chinchilla}, we are interested in the behavior of this procedure for compute scales that are computationally burdensome so we require a method of extrapolating our results from affordable compute scales to much larger ones. We performed extrapolation via two linear fits in \emph{logarithmic space} in which the regressors are the log-optimal training set sizes and the targets are the log-optimal parameter counts. Since the regression takes place in log space, a solution with \emph{slope} $1$ demonstrates a linear relation between optimal parameter count and training samples. As apparent by the noticeable change in slope in Figure \ref{fig:efficient-frontier}, it seems unwise to extrapolate with just a single linear fit of the data in log scale. Therefore, we extrapolate to higher compute scales via a linear fit of only the data from \emph{larger} compute budgets that we tested. Likewise, we extrapolate to lower compute scales via linear fits of only the data from \emph{smaller} compute budgets that we tested. We observed that for large compute budgets, a line of slope $1$ is well within the confidence bounds produced by statistical bootstrapping. This result both lends credence to our theoretical analysis (depicted in orange) and qualitatively matches what was observed by \cite{chinchilla}.

\section{Closing Remarks}

State-of-the-art neural networks have grown tremendously over the past decade. As this trend continues and these neural networks are deployed in production systems it is becoming increasingly important to optimize use of the enormous resources -- computation, energy, data, and human interaction -- that they consume. The traditional approach to neural network modeling, entailing experimentation with many architectures and hyperparameter sweeps before settling on a favorite is no longer viable. This gives rise to demand for a theory that predicts what will work well. This theory might be developed through a combination of mathematical analysis and experimentation with smaller scale models that are much less expensive to train.

Our results represent a first step in developing rigorous mathematics for this purpose. We hope that this will inspire further theoretical research on the subject. Our analysis is based on an error upper bound. While our empirical analysis raises the possibility that this guides reasonable allocations of large compute budgets, further research is required to determine whether this is really the case. Furthermore, our analysis restricts attention to single-hidden-layer feedforward neural networks. Generalizing the results to treat state-of-the-art architectures remains an open issue.

We have only considered allocation of pretraining compute. State-of-the-art performance in modern application domains relies on subsequent fine-tuning (see, e.g., \citep{finetuning2019}) through reinforcement learning from human feedback, and, after that, resources are required in production to support application of the learned model as well as further learning from real-world interactions. How best to allocate resources between pretraining, fine-tuning, and production is another area that deserves attention. An information-theoretic framework that treats pretraining, fine-tuning, and decision making in a unified and coherent manner, perhaps in the vein of \citep{bitbybit2021}, might facilitate theoretical developments on this front.

% Acknowledgments---Will not appear in anonymized version
% \acks{We thank a bunch of people and funding agency.}

\bibliography{coltbib}

\newpage
\appendix

\appendix
\section{Appendix}
\subsection{Proofs of General Information-Theoretic Results}\label{apdx:general}

\begin{theorem}\label{th:decomp1}
    For all $t\in \mathbb{Z}_{+}$ and $\state$,
    $$\E[L_t(P_t)] = \underbrace{\E[-\ln P^*_t]}_{\rm irreducible\ error} + \E[\KL(P^*_t\|\hat{P}_t)] + \underbrace{\E[\KL(\hat{P}_t\|P_t)]}_{\rm inferential\ error}.$$
\end{theorem}
\begin{proof}
    \begin{align*}
        \E[L_t(P_t)]
        & = -\E[\ln P_t]\\
        & = -\E[\ln P_t^*] + \E\left[\ln\frac{P^*_t}{P_t}\right]\\
        & = -\E\left[\ln P_t^*\right] + \E\left[\ln\frac{P^*_t}{\hat{P}_t}\right] + \E\left[\ln\frac{\hat{P}_t}{P_t}\right]\\
        & \overset{(a)}{=} -\E\left[\ln P_t^*\right] + \E\left[\E\left[\ln\frac{P^*_t}{\hat{P}_t}\Big|F,\state,X_t\right]\right] + \E\left[\ln\frac{\hat{P}_t}{P_t}\right]\\
        & \overset{(b)}{=} -\E\left[\ln P_t^*\right] + \E\left[\KL(P^*_t\| \hat{P}_t)\right] + \E\left[\ln\frac{\hat{P}_t}{P_t}\right]\\
        & \overset{(c)}{=} -\E\left[\ln P_t^*\right] + \E\left[\KL(P^*_t\| \hat{P}_t)\right] + \E\left[\E\left[\ln\frac{\hat{P}_t}{P_t}\Big| \state, X_t\right]\right]\\
        & = -\E\left[\ln P_t^*\right] + \E\left[\KL(P^*_t\| \hat{P}_t)\right] + \E\left[\KL(\hat{P}_t\|P_t)\right],
    \end{align*}
    where $(a)$ follows from the law of total expectation, $(b)$ follows from the fact that $Y_{t+1}\perp \state|F, X_t$, and $(c)$ follows from the law of total expectation.
\end{proof}

\begin{lemma}\label{le:miss_est}
    For all $t\in \mathbb{Z}_+$, if $\state$ is produced by an algorithm which satisfies assumptions $(1)$ and $(2)$, then
    $$\E\left[\E[\KL(P^*_t\|\hat{P}_t)]\right] = \underbrace{\I(Y_{t+1};F|\finfunc, \state, X_t)}_{\rm misspecification\ error} + \underbrace{\I(Y_{t+1};\finfunc|\state, X_t)}_{\rm estimation\ error}.$$
\end{lemma}
\begin{proof}
    \begin{align*}
        \E\left[\E[\KL(P^*_t\|\hat{P}_t)]\right]
        & = \I(Y_{t+1};F|\state, X_t)\\
        & \overset{(a)}{=} \I(Y_{t+1};F|\state, X_t) + \I(Y_{t+1};\finfunc|F, \state, X_t)\\
        & = \I(Y_{t+1};F,\finfunc|\state, X_t)\\
        & = \I(Y_{t+1};F|\finfunc, \state, X_t) + \I(Y_{t+1};\finfunc|\state, X_t),
    \end{align*}
    where $(a)$ follows from assumption $(2)$: Let $\tilde{H}_{t+1:\infty}$ be a sequence of iid data pairs which is generated by $F$. If we let $\tilde{H}_{\infty} = (H_t, \tilde{H}_{t+1:\infty})$, then by assumption $(2)$, $\finfunc = \lim_{t\to\infty} F_{\tilde{U}_t}$ where $\tilde{U}_t$ is the result of running our algorithm on $\tilde{H}_t$ from $\initstate$. Note that this is equivalent to running our algorithm on $H_{t+1:\infty}$ starting from $\state$. As a result, $\finfunc$ is deterministic when conditioned on $(F, \state, X_t)$, so $\I(Y_{t+1};\finfunc|F, \state, X_t) = 0$
\end{proof}

\lossDecomp*
\begin{proof}
    The result follows from Theorem \ref{th:decomp1} and Lemma \ref{le:miss_est}
\end{proof}

\begin{lemma}{\bf (misspecification error upper bound)}
    \label{le:miss_error_bound}
    For all $t\in \mathbb{Z}_+$, if $\state$ is produced by an algorithm which satisfies assumptions $(1)$ and $(2)$, then
    $$\I(Y_{t+1};F|\finfunc, \state, X_t) \leq \I(Y_{t+1};F|\finfunc, \initstate, X_t).$$
\end{lemma}
\begin{proof}
    \begin{align*}
        \I(Y_{t+1};F|\finfunc, \state, X_t)
        & = \I(Y_{t+1};F, \finfunc, \finfunc, X_t) - \I(Y_{t+1};\finfunc, \state, X_t)\\
        & \overset{(a)}{=} \I(Y_{t+1};F, \finfunc, \initstate, X_t) - \I(Y_{t+1}; \finfunc, \state, X_t)\\
        & = \I(Y_{t+1};F, \finfunc, \initstate, X_t) - \I(Y_{t+1};\finfunc, X_t) - \I(Y_{t+1};\state|\finfunc, X_t)\\
        & \overset{(b)}{\leq} \I(Y_{t+1};F, \finfunc, \initstate, X_t) - \I(Y_{t+1};\finfunc, X_t) - \I(Y_{t+1};\initstate|\finfunc, X_t)\\
        & = \I(Y_{t+1};F, \finfunc, \initstate, X_t) - \I(Y_{t+1};\finfunc, \initstate, X_t)\\
        & = \I(Y_{t+1};F|\finfunc, \initstate, X_t),
    \end{align*}
    where $(a)$ follows from assumption $(2)$ and $(b)$ follows from the fact that $Y_{t+1}\perp \initstate|(\state, X_t, \finfunc)$.
\end{proof}

\begin{lemma}{\bf (estimation error upper bound)}
    \label{le:est_error_bound}
    For all $t,\mathbb{Z}_{+}$, if $\state$ is produced by an algorithm which satisfies assumptions $(1)$ and $(2)$, then for all $T \in \mathbb{Z}_{++}$,
    $$\frac{1}{T}\sum_{t=0}^{T-1}\I(Y_{t+1};\finfunc|\state, X_t) \leq \frac{\H(\finfunc)}{T}.$$
\end{lemma}
\begin{proof}
    \begin{align*}
        \frac{1}{T}\sum_{t=0}^{T-1}\I(Y_{t+1};\finfunc|\state, X_t)
        & = \frac{1}{T}\sum_{t=0}^{T-1}\H(\finfunc|\state, X_t) - \H(F_{\theta_{\infty}}|\state, X_t, Y_{t+1})\\
        & \leq \frac{1}{T}\sum_{t=0}^{T-1}\H(\finfunc|\state, X_t) - \H(F_{\theta_{\infty}}|\state, X_t, Y_{t+1}, U_{t+1})\\
        & \overset{(a)}{=} \frac{1}{T}\sum_{t=0}^{T-1}\H(\finfunc|\state, X_t) - \H(F_{\theta_{\infty}}|U_{t+1})\\
        & \leq \frac{1}{T}\sum_{t=0}^{T-1}\H(\finfunc|\state) - \H(F_{\theta_{\infty}}|U_{t+1})\\
        & = \frac{\H(\finfunc|\initstate)- \H(\finfunc|U_{T})}{T}\\
        & \overset{(b)}{=} \frac{\H(\finfunc|\initstate)- \H(\finfunc|U_{T}, \initstate)}{T}\\
        & = \frac{\I(\finfunc;U_{T}|\initstate)}{T}\\
        & \overset{(c)}{\leq} \frac{\H(\finfunc|\initstate)}{T}\\
        & \leq \frac{\H(\finfunc)}{T},
    \end{align*}
    where $(a)$ follows from the fact that $\finfunc\perp (\state, X_t, Y_{t+1})|(U_{t+1})$, $(b)$ follows from the fact that $\finfunc \perp \initstate|U_{T}$.
\end{proof}

\begin{theorem}{\bf(reducible error upper bound)}
    For all $t\in \mathbb{Z}_+$, if $\state$ is produced by an algorithm which satisfies assumptions $(1)$, $(2)$, and $(3)$ then,
    $$\E[\KL(P^*_t\|\hat{P}_t)] \leq \I(Y_{t+1};F|\finfunc, \initstate, X_t) + \frac{\H(\finfunc)}{t}.$$
\end{theorem}
\begin{proof}
    \begin{align*}
        \E[\KL(P^*_t|\hat{P}_t)]
        & \overset{(a)}{\leq} \frac{1}{T}\sum_{t=0}^{T-1}\E[\KL(P^*_t|\hat{P}_t)]\\
        & \overset{(b)}{=} \frac{1}{T}\sum_{t=0}^{T-1}\left[\I(Y_{t+1};F|\finfunc, \state, X_t)+\I(Y_{t+1};\finfunc|\state, X_t)\right]\\
        & \overset{(c)}{\leq} \I(Y_{t+1};F|F_{\theta_{\infty}}, \initstate, X_t) + \frac{1}{T}\sum_{t=0}^{T-1}\I(Y_{t+1};\finfunc|\state, X_t)\\
        & \overset{(d)}{\leq} \I(Y_{t+1};F|F_{\theta_{\infty}}, \initstate, X_t) + \frac{\H(\finfunc)}{T},
    \end{align*}
    where $(a)$ follows from assumption $(3)$, $(b)$ follows Lemma \ref{le:miss_est}, $(c)$ follows from Lemma \ref{le:miss_error_bound} and $(d)$ follows from Lemma \ref{le:est_error_bound}.
\end{proof}

\section{Proofs of Misspecification Error Bounds}
\subsection{General Lemmas}
We begin this section with two very general lemmas that facilitate analysis. We first introduce some notation that is necessary to parse the results of this section. Consider a function $f(z) = \Pr(G|Z=z)$ where $G$ is an event. Given another random variable $Y$ with the same range as $Z$, we use the assignment operator $\leftarrow$ to denote $f(Y)$ by $\Pr(G|Z \leftarrow Y)$. Note that in general $\Pr(G|Z\leftarrow Y)$ differs from $\Pr(G|Z = Y)$ as the latter conditions on the event $Z = Y$ while the former represents a change of measure for the variable $Z$.

The first result provides an upper bound on the error of the predictor $\Pr(Y_{t+1}\in\cdot|\tilde{F}, X_t)$ via the error of the predictor that takes $\tilde{F}$ to be $F$: $\Pr(Y_{t+1}\in\cdot|F\leftarrow \tilde{F}, X_t)$.

\begin{lemma}
    \label{le:plugInEst}{\bf (upper bounds via change of measure)}
    For all $d, K, M \in \Z_{++}$, and any learning model $\finfunc$,
    $$\E\left[\KL(P_t^*\|\tilde{P}_t)\right] \leq \E\left[\KL(P^*_t\|\bar{P}_t)\right],$$
    where $P_t^* = \Pr(Y_{t+1}\in\cdot|F, X_t),\ \tilde{P}_t = \Pr(Y_{t+1}\in\cdot|\finfunc,\initstate, X_t),$ and $\bar{P}_t = \Pr(Y_{t+1}\in\cdot|F\leftarrow \finfunc, X_t)$.
\end{lemma}
\begin{proof}
    \begin{align*}
        & \E\left[\KL\left(P^*_t,\|\bar{P}_t\right)\right]\\
        & = \E\left[\KL\left(\Pr(Y_{t+1}\in\cdot|F, \finfunc, \initstate, X_t)\|\bar{P}_t\right)\right]\\
        & = \E\left[\E\left[\ln\frac{\tilde{P}_t}{P^*_t}\bigg|F, \finfunc, \initstate, X_t\right]\right] + \E\left[\E\left[\ln\frac{\bar{P}_t}{\tilde{P}_t}\bigg|F, \finfunc, \initstate, X_t\right]\right]\\
        & = \E\left[\E\left[\ln\frac{\tilde{P}_t}{P^*_t}\bigg|F, \finfunc, \initstate, X_t\right]\right] + \E\left[\E\left[\E\left[\ln\frac{\bar{P}_t}{\tilde{P}_t}\bigg|F, \finfunc, \initstate, X_t\right]\bigg|F_{\theta_{\infty}}, \initstate, X_t\right]\right]\\
        & = \E\left[\KL\left(P^*_t\|\tilde{P}_t\right)\right] + \E\left[\E\left[\ln\frac{\bar{P}_t}{\tilde{P}_t}\bigg|\finfunc, \initstate,  X_t\right]\right]\\
        & = \E\left[\KL\left(P^*_t\|\tilde{P}_t\right)\right] + \E\left[\KL(\tilde{P}_t\|\bar{P}_t)\right]\\
        & \geq \E\left[\KL\left(P^*_t\|\tilde{P}_t\right)\right].
    \end{align*}
\end{proof}
This lemma states that if we have a learning model $\finfunc$, then the performance of a predictor that interprets $\finfunc$ as the \emph{true} function $F$ upper bounds the performance of an optimal predictor reasons about the conditional distributions $\Pr(F\in\cdot|\finfunc, \initstate)$ and $\Pr(Y\in\cdots|F, X)$. This result is helpful in situations in which the conditional distribution $\Pr(F\in\cdot|\finfunc,\initstate)$ is difficult to analyze. 

\begin{lemma}\label{le:kl_ub}{\bf(squared error upper bounds KL)}
    For all real-valued random variables $G$ and $\tilde{G}$, if $Y$ is a binary random variable for which
    $\Pr(Y=1|G) = \frac{1}{1+e^{-G}}$, then 
    $$\E\left[\KL(\Pr(Y\in\cdot|G)\| \Pr(Y\in\cdot|G\leftarrow\tilde{G}))\right] \leq \E\left[\left(G-\tilde{G} \right)^2\right].$$
\end{lemma}
\begin{proof}
    \begin{align*}
        \E\left[\KL(\Pr(Y\in\cdot|G)\| \Pr(Y\in\cdot|G\leftarrow\tilde{G}))\right]
        & = \E\left[\frac{1}{1+e^{G}}\ln\left(\frac{1+e^{\tilde{G}}}{1+e^{G}}\right)\right]\\
        & \quad + \E\left[\frac{1}{1+e^{-G}}\ln\left(\frac{1+e^{-\tilde{G}}}{1+e^{-G}}\right)\right]\\
        & \overset{(a)}{\leq} \E\left[\left(G - \tilde{G}\right)^2\right]\\
    \end{align*}
    where $(a)$ follows from the fact that for all $x, y \in \Re$, $\frac{1}{1+e^x}\ln\left(\frac{1+e^y}{1+e^x}\right) + \frac{1}{1+e^{-x}}\ln\left(\frac{1+e^{-y}}{1+e^{-x}}\right) \leq (x-y)^2$.
\end{proof}

\subsection{Lemmas Pertaining to Dirichlet Multinomial Random Variables}
In this section, we study a finite-dimensional analogue of the Dirichlet Process. Let $d, K, N\in\mathbb{Z}_{++}$ where $\theta'\in\Re^{N}$ are the hidden unit weights and $A \in \Re^{N\times d}$ where each row $A_i$ is distributed uniformly on $\sphere$. Then,
$$F_N(X) = \frac{1}{1+\exp\left(-\sqrt{K+1}\theta'^\top \relu(AX)\right)}.$$
We now describe how $\theta'$ is generated. First, $\bar{\theta}'$ is sampled from ${\rm Dirichlet}([K/N, \ldots, K/N])$. For each $i \in [N]$, we randomly flip the sign of the weight according to
$$\theta_i' = \left\{\begin{array}{ll}
\overline{\theta}_i' \qquad & \text{with probability } 1/2, \\
-\overline{\theta}_i' \qquad & \text{otherwise.}
\end{array}\right.$$
We let $Y'$ be a binary random variable for which $\Pr(Y'=1|F_N, X) = F_N(X).$

We now describe the learning model. For all $n\in\mathbb{Z}_{++}$ and $\delta > 0$, our learning model weights $w_1, \ldots, w_n$ are sampled iid from a categorical distribution with probabilities specified by $\bar{\theta}'$.
$$\finfunc^N = \frac{1}{1+\exp\left(- \frac{\sqrt{K+1}}{n} \sum_{i=1}^n \sign(\theta_{w_i}) \relu(\tilde{A}_{w_i}^\top X)\right)},$$
where $\tilde{A}_{w_i}$ is a quantization of $A_{w_{i}}$ with tolerance $\delta$.  For notational simplicity in the results to follow, we will choose the following reparameterization:
$$\finfunc^N = \frac{1}{1+\exp\left(- \sqrt{K+1} \proxytheta'^\top \relu(\tilde{A}X)\right)}$$
where $\proxytheta'_i = \frac{\sign(\theta_i)\sum_{j=1}^{n}\mathbbm{1}_{[w_j=i]}}{n}$ and $\tilde{A} \in \Re^{N\times d}$ is the matrix for which row $i$ is $\tilde{A}_i$.

We first derive the mean squared error between the logits of $F_N$ and $\tilde{F}_{n,t}'$.
\begin{lemma}
    \label{le:dir_mult_ub}
    For all $d, n, N\in \mathbb{Z}_{++}$, if $X\sim\normal(0, I_d)$, then
    $$\E\left[\left(\left(\theta' - \proxytheta'\right)\relu(AX)\right)^2\right] \leq \frac{1}{n}.$$
\end{lemma}
\begin{proof}
    \begin{align*}
        \E\left[\left(\left(\theta' - \frac{1}{n}\sign(\theta')\odot\proxytheta'\right)\relu(AX)\right)^2\right]
        & = \E\left[\relu(AX)^\top\left(\proxytheta'\proxytheta'^\top - \theta'\theta'^\top\right)\relu(AX)\right]\\
        & \leq \E\left[\relu(AX)^\top\left(\proxytheta'\proxytheta'^\top\right)\relu(AX)\right]\\
        & = \E\left[\relu(AX)^\top\left(\frac{{\rm diag}(\bar{\theta}')}{n}\right)\relu(AX)\right]\\
        & \leq \E\left[\frac{1}{n}\sum_{i=1}^{N}|\theta_i'|(A_i^\top X)^2\right]\\
        & = \frac{1}{n}
    \end{align*}
\end{proof}

With this result, we are able to derive the following two upper bounds on misspecification error.

\begin{lemma}\label{le:dir_mult_miss_ub2}
    For all $d, n, K, N \in \mathbb{Z}_{++}$ and $\delta > 0$, if $X\sim\normal(0,I_d)$, then
    $$\I(Y';A|\theta', \proxytheta', \tilde{A}, X) \leq \frac{2(K+1)}{n} + 2d\delta^2\quad \text{and}\quad  \I(Y';\theta'|\proxytheta', A, X) \leq \frac{K+1}{n}.$$
\end{lemma}
\begin{proof}
    \begin{align*}
        & \I(Y';A|\theta', \proxytheta', \tilde{A}, X)\\
        & = \E\left[\KL(\Pr(Y'\in\cdot|\theta', A, X) \| \Pr(Y'\in\cdot|\theta', \proxytheta', \tilde{A}, X))\right]\\
        & \overset{(a)}{\leq}\E\left[\KL(\Pr(Y'\in\cdot|\theta', A, X) \| \Pr(Y'\in\cdot|\theta', \proxytheta', A\leftarrow\tilde{ A}, X))\right]\\
        & \overset{(b)}{\leq}\E\left[(K+1)\left(\theta'^\top\relu(A^\top X) - \theta'^\top\relu(\tilde{A}^\top X)\right)^2\right]\\
        & \overset{(c)}{=} \E\left[(K+1)\left(\theta'^\top\relu(A^\top X)-\hat{\theta}'^\top\relu(A^\top X) + \hat{\theta}'^\top\relu(A^\top X) - \theta'^\top\relu(\tilde{A}^\top X)\right)^2\right]\\
        & \leq 2(K+1)\cdot\E\left[\left(\theta'^\top\relu(A^\top X)-\hat{\theta}'^\top\relu(A^\top X)\right)^2 \right]\\
        &\quad + 2(K+1)\cdot\E\left[\left( \hat{\theta}'^\top\relu(A X) - \theta'^\top\relu(\tilde{A} X)\right)^2\right]\\
        & \overset{(d)}{\leq} \E\left[2(K+1)\left(\frac{1}{n} + \left( \hat{\theta}'^\top\relu(\hat{A} X) - \hat{\theta}'^\top\relu(\tilde{A} X)\right)^2\right)\right]\\
        & \leq \E\left[2(K+1)\left(\frac{1}{n} + \left(\theta'^\top\relu(\hat{A} X) - \theta'^\top\relu(\tilde{A} X)\right)^2\right)\right]\\
        & = \E\left[2(K+1)\left(\frac{1}{n} + \left(\relu(\hat{A}X) - \relu(\tilde{A}X)\right)^\top \theta'\theta'^\top\left(\relu(\hat{A}X) - \relu(\tilde{A}X)\right)\right)\right]\\
        & = \frac{2(K+1)}{n}+2(K+1)\E\left[\left(\relu(\hat{A}X) - \relu(\tilde{A}X)\right)^\top \theta'\theta'^\top\left(\relu(\hat{A}X) - \relu(\tilde{A}X)\right)\right]\\
        & = \frac{2(K+1)}{n}+2(K+1)\E\left[\sum_{i=1}^{N}\E[(\theta'_i)^2|\hat{A}, \tilde{A}, X]\cdot(\relu(\hat{A}_i^\top X)-\relu(\tilde{A}_i^\top X))^2\right]\\
        & \leq \frac{2(K+1)}{n}+2(K+1)\E\left[\sum_{i=1}^{N}\E[(\theta'_i)^2|\hat{A}, \tilde{A}, X]\cdot(\hat{A}_i^\top X-\tilde{A}_i^\top X)^2\right]\\
        & \leq \frac{2(K+1)}{n}+2(K+1)\E\left[\sum_{i=1}^{N}\E[(\theta'_i)^2|\hat{A}] \delta^2 \|(X)\|_2^2\right]\\
        & \leq \frac{2(K+1)}{n}+2(K+1)\E\left[\sum_{i=1}^{N}\frac{1-\frac{1}{N}}{N(K+1)}\cdot d\delta^2\right]\\
        & \leq \frac{2(K+1)}{n}+2d\delta^2
    \end{align*}
    where $(a)$ follows from Lemma \ref{le:plugInEst}, $(b)$ follows from Lemma \ref{le:kl_ub}, where in $(c)$, 
    $$\hat{\theta}_i' = 
        \begin{cases}
            \theta_i' & \text{ if } \proxytheta_i' > 0\\
            0 & \text{ otherwise}
        \end{cases},
    $$
    and $(d)$ follows from Lemma \ref{le:dir_mult_ub} and 
    $$\hat{A}_i = 
        \begin{cases}
            A_i & \text{ if } \proxytheta_i' > 0\\
            0 & \text{ otherwise}
        \end{cases}.
    $$
    \begin{align*}
        \I(Y';\theta'|\proxytheta', A, X)
        & = \E\left[\KL(\Pr(Y'\in\cdot|\theta', A, X) \| \Pr(Y'\in\cdot|\proxytheta', A, X))\right]\\
        & \overset{(a)}{\leq}\E\left[\KL(\Pr(Y'\in\cdot|\theta', A, X) \| \Pr(Y'\in\cdot|\theta'\leftarrow\proxytheta', A, X))\right]\\
        & \overset{(b)}{\leq}\E\left[(K+1)\left(\theta'^\top\relu(A^\top X) - \proxytheta'^\top\relu(A^\top X)\right)^2\right]\\
        & \overset{(c)}{\leq}\frac{K+1}{n},
    \end{align*}
    where $(a)$ follows from Lemma \ref{le:plugInEst}, $(b)$ follows from Lemma \ref{le:kl_ub}, and $(c)$ follows from Lemma \ref{le:dir_mult_ub}.
\end{proof}

Notably, all of these results are \emph{independent} of $N$. Therefore, with an appropriate application of dominated convegence, these results will hold as $N\rightarrow \infty$ and $F_N \overset{d}{\rightarrow} F$.

\subsection{From Dirichlet Random Variable to Dirichlet Process}
In this section, we will extend the results of the previous section to our Dirichlet process. We will define a few new variables for our learning model up front. Recall that our learning model $\finfunc$ samples $w_1, \ldots, w_n$ iid according to $\bar{\theta}$. Then,
$$\finfunc(X) = \frac{1}{1+\exp\left(- \frac{\sqrt{K+1}}{n} \sum_{i=1}^n \sign(\theta_{w_i}) \relu(\bar{w}_{i,\delta}^\top X)\right)},$$
where $\bar{w}_{i,\delta}$ is a quantization of $w_i$ for which $\|w_i - \bar{w}_{i,\delta}\|_2 \leq \delta$. $\finfunc$ can be expressed in the following alternate way:
$$\tilde{F}_{n,t}(X) = \frac{1}{1+\exp\left(- \sqrt{K+1} \sum_{w\in\mathcal{W}} \proxytheta_w \relu(\bar{w}_{\delta}^\top X)\right)},$$
where $\proxytheta_w = \frac{\sign(\theta_w)\cdot\sum_{i=1}^{n}\mathbbm{1}_{[w_i=w]}}{n}$ and $\bar{w}_{\delta}$ is the quantization of $w$ with tolerance $\delta$. We let $\tilde{\mathcal{W}} = \{w\in\mathcal{W}: |\proxytheta|>0 \}$ and $\bar{\mathcal{W}}_\delta = \{\bar{w}_\delta : w\in\tilde{\mathcal{W}}\}$.

We begin by extending Lemma \ref{le:dir_mult_miss_ub2} to our Dirichlet process and learning model.
\begin{corollary}\label{cor:miss_ub_1}
    For all $d, n, K\in\mathbb{Z}_{++}$,
    $$\I(Y;\theta|\proxytheta, \mathcal{W}, X) \leq \frac{K+1}{n}.$$
\end{corollary}
\begin{proof}
    \begin{align*}
        \I(Y;\theta|\proxytheta, \mathcal{W}, X)
        & = \E\left[\KL(\Pr(Y\in\cdot|\theta, \mathcal{W}, X)\|\Pr(Y\in\cdot|\proxytheta, \mathcal{W}, X))\right]\\
        & \overset{(a)}{\leq} \E\left[\KL(\Pr(Y\in\cdot|\theta, \mathcal{W}, X)\|\Pr(Y\in\cdot|\theta\leftarrow\proxytheta, \mathcal{W}, X))\right]\\
        & \overset{(b)}{\leq} (K+1)\E\left[\left(\sum_{w\in\mathcal{W}}\theta_w\relu(w^\top X) - \sum_{w\in\mathcal{W}}\proxytheta_w\relu(w^\top X)\right)^2\right]\\
        & \overset{(c)}{=} (K+1)\E\left[\left(\lim_{N\rightarrow\infty} \sum_{i=1}^{N}\theta_i'\relu(A_i^\top X) - \sum_{i=1}^{N}\proxytheta_i'\relu(A_i^\top X)\right)^2\right]\\
        & \overset{(d)}{=} \lim_{N\rightarrow\infty}(K+1)\E\left[\left(\sum_{i=1}^{N}\theta_i'\relu(A_i^\top X) - \sum_{i=1}^{N}\proxytheta_i'\relu(A_i^\top X)\right)^2\right]\\
        & \overset{(e)}{\leq} \lim_{N\rightarrow\infty}\frac{K+1}{n}\\
        & = \frac{K+1}{n},\\
    \end{align*}
    where $(a)$ follows from Lemma \ref{le:plugInEst}, $(b)$ follows from Lemma \ref{le:kl_ub}, $(c)$ follows from the fact that the distribution of $\theta$ is the limiting distribution $\lim_{N\rightarrow\infty}$ of a Dirichlet $[K/N,\ldots, K/N]$ random variable, $(d)$ follows from the dominated convergence theorem, and $(e)$ follows from Lemma \ref{le:dir_mult_ub}. 
\end{proof}
This result is intuitive as the results shown for the Dirichlet distribution and multinomial learning model in the previous section were independent of the dimension $N$. The following corollary bounds the remaining term from Lemma \ref{le:dir_mult_miss_ub2} for our data generating process and learning model.

\begin{corollary}\label{cor:miss_ub_2}
    For all $d, n, K\in\mathbb{Z}_{++}$, and $\delta > 0$, if $\tilde{\mathcal{W}} = \{w\in\mathcal{W}: |\proxytheta_w| > 0\}$, and $\bar{\mathcal{W}} = \{w+Z_w : w\in\tilde{\mathcal{W}}\}$ where $Z_w\overset{iid}{\sim}\normal(0, \delta^2I_d)$, then
    $$\I(Y;\mathcal{W}|\theta, \proxytheta, \bar{\mathcal{W}}, X) \leq \frac{2(K+1)}{n} + 2d\delta^2.$$
\end{corollary}
\begin{proof}
    \begin{align*}
        \I(Y;\mathcal{W}|\theta, \proxytheta, \bar{\mathcal{W}}, X)
        & = \E\left[\KL(\Pr(Y\in\cdot|\theta, \mathcal{W}, X)\|\Pr(Y\in\cdot|\theta, \proxytheta, \bar{\mathcal{W}}, X))\right]\\
        & \overset{(a)}{\leq} \E\left[\KL(\Pr(Y\in\cdot|\theta, \mathcal{W}, X)\|\Pr(Y\in\cdot|\theta, \proxytheta, \mathcal{W}\leftarrow\bar{\mathcal{W}}, X))\right]\\
        & \overset{(b)}{\leq}  \E\left[(K+1)\left(\sum_{w\in\mathcal{W}} \theta_w \relu(w^\top X_t) - \sum_{w\in\tilde{\mathcal{W}}} \theta_w \relu((w+Z_w)^\top X_t)\right)^2\right]\\
        & \overset{(c)}{=} \E\left[(K+1)\left(\lim_{N\rightarrow\infty}\sum_{i=1}^{N} \theta_i' \relu(A_i^\top X_t) - \sum_{i=1}^{N} \theta'_i \relu(\tilde{A}_i^\top X_t)\right)^2\right]\\
        & \overset{(d)}{=} \lim_{N\rightarrow\infty}\E\left[(K+1)\left(\sum_{i=1}^{N} \theta_i' \relu(A_i^\top X_t) - \sum_{i=1}^{N} \theta'_i \relu(\tilde{A}_i^\top X_t)\right)^2\right]\\
        & \overset{(e)}{\leq} \lim_{N\rightarrow\infty} \frac{2(K+1)}{n} + 2d\delta^2\\
        & = \frac{2(K+1)}{n} + 2d\delta^2,
    \end{align*}
    where $(a)$ follows from Lemma \ref{le:plugInEst}, $(b)$ follows from Lemma \ref{le:kl_ub}, $(c)$ follows from the fact that the distribution of $\theta$ is the limiting distribution $\lim_{N\rightarrow\infty}$ of a Dirichlet $[K/N,\ldots, K/N]$ random variable, $(d)$ follows from the dominated convergence theorem, and $(e)$ follows from Lemma \ref{le:dir_mult_miss_ub2}. 
\end{proof}

The final step before proving Theorem \ref{th:mis_error_ub} is to show that the mutual information between $Y$ and $\theta$ is \emph{higher} when we condition on the true basis functions $\tilde{W}$ as opposed to the noisy ones $\bar{\mathcal{W}}$.
\begin{lemma}{\bf(more is learned with the true input)}
    For all $n, d, K \in \mathbb{Z}_{++}$,
    \label{le:true_input}
    $$\I(Y;\theta|\proxytheta, \bar{\mathcal{W}}, X) \leq \I(Y;\theta|\proxytheta, \mathcal{W}, X).$$
\end{lemma}
\begin{proof}
    \begin{align*}
        \I(Y;\theta|\proxytheta, \bar{\mathcal{W}}, X)
        & = \I(Y, X, \bar{\mathcal{W}};\theta|\proxytheta) - \I(X, \bar{\mathcal{W}};\theta|\proxytheta)\\
        & = \I(Y, X, \bar{\mathcal{W}};\theta|\proxytheta)\\
        & = \I(Y;\theta|\proxytheta) + \I(X, \bar{\mathcal{W}};\theta|\proxytheta, Y)\\
        & \overset{(a)}{\leq} \I(Y;\theta|\proxytheta) + \I(X, \mathcal{W};\theta|\proxytheta, Y)\\
        & = \I(Y, X, \mathcal{W};\theta|\proxytheta)\\
        & = \I(Y;\theta|\proxytheta,\mathcal{W}, X) + \I(\mathcal{W},X;\theta|\proxytheta)\\
        & = \I(Y;\theta|\proxytheta,\mathcal{W}, X),
    \end{align*}
    where $(a)$ follows from the fact that $\theta\perp\bar{\mathcal{W}}|(X,Y,\proxytheta, \mathcal{W})$ and the data processing inequality.
\end{proof}
With this result in place, we provide a proof for Theorem \ref{th:mis_error_ub}.
\misspecificationUb*
\begin{proof}
    \begin{align*}
        \I(Y_{t+1};F|\tilde{F}_{n,t}, X_t)
        & \overset{(a)}{=} \I(Y_{t+1};\theta, \mathcal{W}|\tilde{\theta}, \tilde{\mathcal{W}}, X_t)\\
        & \overset{(b)}{=} \I(Y_{t+1};\theta|\tilde{\theta}, \tilde{\mathcal{W}}, X_t) + \I(Y_{t+1};\mathcal{W}|\theta, \tilde{\theta}, \tilde{\mathcal{W}}, X_t)\\
        & \overset{(c)}{\leq} \I(Y_{t+1};\theta|\tilde{\theta}, \mathcal{W}, X_t) + \I(Y_{t+1};\mathcal{W}|\theta, \tilde{\theta}, \tilde{\mathcal{W}}, X_t)\\
        & \overset{(d)}{\leq} \frac{(K+1)}{n} + \frac{2(K+1)}{n} + 2d\delta^2\\
        & = \frac{3(K+1)}{n} + 2d\delta^2
    \end{align*}
    where in $(a)$, $w_1,\ldots, w_n \overset{iid}{\sim} \bar{\theta}$ and $\proxytheta = \frac{1}{n}\sum_{i=1}^{n}\sign(\theta_{w_i})$ and $\tilde{\mathcal{W}} = \{w\in\mathcal{W}: |\proxytheta_w| > 0\}$, $(b)$ follows from the chain rule of mutual information, $(c)$ follows from Lemma \ref{le:true_input}, and $(d)$ follows from Corollaries \ref{cor:miss_ub_1} and \ref{cor:miss_ub_2}.
\end{proof}

\section{Proofs of Estimation Error Bounds}
In this section, we will provide proofs for the estimation error bound shown in Theorem \ref{th:est_error_bound}. The roadmap will match that of the previous section. We will first prove results for the Dirichlet-Multinomial process with finite dimension $N$ and then appropriately apply dominated convergence as $N\rightarrow \infty$ to derive the results for our Dirichlet process and learning model.
\subsection{General Lemmas}
The following lemmas are general combinatorics results that we will eventually apply to bound the estimation error of the finite $N$ Dirichlet-Multinomial process.
\begin{lemma}\label{le:multinomial_terms}
    For all $m, j, n, K, N\in\mathbb{Z}_{++}$ s.t. $j < n$, if 
    $$C_m(j) = \underbrace{\sum_{i=j}^{n-1}\frac{\frac{K}{N}}{K+i}\cdot\sum_{i=j+1}^{n-1}\frac{\frac{K}{N}}{K+i}\cdot \ldots\cdot\sum_{i > j+m-1}^{n-1}\frac{\frac{K}{N}}{K+i}}_{m},$$
    then
    $$\frac{1}{m!}\frac{K^m}{N^m}\ln^m\left(\frac{K+n}{K+m-1+j}\right) \leq C_m(j) \leq \frac{1}{m!}\frac{K^m}{N^m}\ln^m\left(\frac{K+n}{K-1+j}\right).$$
\end{lemma}
\begin{proof}
    We first prove the upper bound via induction. Base Case: $n=1$
    \begin{align*}
        C_1(j) 
        & = \sum_{i=j}^{n-1}\frac{\frac{K}{N}}{K+i}\\
        & \leq \frac{K}{N}\int_{j-1}^{n}\frac{1}{K+x}dx\\
        & = \frac{K}{N}\ln\left(\frac{K+n}{K-1+j}\right)
    \end{align*}
    Assume inductive hypothesis is true for $m=k$.
    \begin{align*}
        C_{k+1}(j) 
        & = \sum_{i=j}^{n-1}\frac{\frac{K}{N}}{K+i} \cdot C_k(i+1)\\
        & \leq \frac{K}{N}\int_{j-1}^{n}\frac{1}{K+x}\cdot C_k(x+1)\\
        & \leq \frac{K^{k+1}}{N^{k+1}}\frac{1}{k!}\int_{j-1}^{n}\frac{1}{K+x}\cdot \ln^k\left(\frac{K+n}{K+x}\right)dx\\
        & \leq -\frac{K^{k+1}}{N^{k+1}}\frac{1}{(k+1)!}\cdot \ln^{k+1}\left(\frac{K+n}{K+x}\right)\bigg|^{n}_{j-1}\\
        & = \frac{K^{k+1}}{N^{k+1}}\frac{1}{(k+1)!}\ln^{k+1}\left(\frac{K+n}{K-1+j}\right).
    \end{align*}
    We now prove the lower bound also via induction. Base Case: $m=1$
    \begin{align*}
        C_1(j) 
        & = \sum_{i=j}^{n-1}\frac{\frac{K}{N}}{K+i}\\
        & \geq \frac{K}{N}\int_{j}^{n}\frac{1}{K+x}dx\\
        & = \frac{K}{N}\ln\left(\frac{K+n}{K+j}\right)
    \end{align*}
     Assume inductive hypothesis is true for $m=k$.
    \begin{align*}
        C_{k+1}(j) 
        & = \sum_{i=j}^{n-1}\frac{\frac{K}{N}}{K+i} \cdot C_k(i+1)\\
        & \geq \frac{K}{N}\int_{j}^{n}\frac{1}{K+x}\cdot C_k(x+1)\\
        & \geq \frac{K^{k+1}}{N^{k+1}}\frac{1}{k!}\int_{j}^{n}\frac{1}{K+x}\cdot \ln^k\left(\frac{K+n}{K+k+x}\right)dx\\
        & \geq -\frac{K^{k+1}}{N^{k+1}}\frac{1}{(k+1)!}\cdot \ln^{k+1}\left(\frac{K+n}{K+k+x}\right)\bigg|^{n}_{j}\\
        & =\frac{K^{k+1}}{N^{k+1}}\frac{1}{(k+1)!}\ln^{k+1}\left(\frac{K+k+n}{K+k+j}\right)\\
        & \geq \frac{K^{k+1}}{N^{k+1}}\frac{1}{(k+1)!}\ln^{k+1}\left(\frac{K+n}{K+k+j}\right).
    \end{align*}
    The result follows.
\end{proof}

\begin{lemma}\label{le:alt_terms}
    For all $i, n, K, N\in\mathbb{Z}_{++}$, if $2 \leq K \leq \sqrt{N}$ and $n \leq N$,  then
    $$\frac{1}{(2i)!}\frac{K^{2i}}{N^{2i}}\ln^{2i}\left(\frac{K+n}{K-1}\right)-\frac{1}{(2i+1)!}\frac{K^{2i+1}}{N^{2i+1}}\ln^{2i+1}\left(\frac{K+n}{K+2i}\right) \geq 0.$$
\end{lemma}
\begin{proof}
    \begin{align*}
        0
        & \overset{(a)}{\leq} 1 - \frac{1}{2i+1}\frac{K}{N}\ln\left(\frac{K+n}{K-1}\right)\\
        & = \ln^{2i}\left(\frac{K+n}{K-1}\right)-\frac{1}{2i+1}\frac{K}{N}\ln^{2i+1}\left(\frac{K+n}{K-1}\right)\\
        & \leq \ln^{2i}\left(\frac{K+n}{K-1}\right)-\frac{1}{2i+1}\frac{K}{N}\ln^{2i+1}\left(\frac{K+n}{K+2i}\right)\\
        & = \frac{1}{(2i)!}\frac{K^{2i}}{N^{2i}}\ln^{2i}\left(\frac{K+n}{K-1}\right)-\frac{1}{(2i+1)!}\frac{K^{2i+1}}{N^{2i+1}}\ln^{2i+1}\left(\frac{K+n}{K+2i}\right)\\
    \end{align*}
    where $(a)$ follows for all $n \leq \left(K-1\right)e^{\frac{3N}{K}}-K$ which is implied by $n \leq N$ for $K \geq 2$.
\end{proof}

\begin{lemma}\label{le:multi_lb}
    For all $n, K, N\in\mathbb{Z}_{++}$, if $2 \leq K \leq \sqrt{N}$ and $n \leq N$, then 
    $$\prod_{i=0}^{n-1}\left(1 - \frac{\frac{K}{N}}{K+i}\right) \geq 1 - \frac{K}{N}\ln\left(1 + \frac{n}{K}\right).$$
\end{lemma}
\begin{proof}
    \begin{align*}
        \prod_{i=0}^{n-1}\left(1 - \frac{\frac{K}{N}}{K+i}\right)
        & \overset{(a)}{\geq} 1 - \sum_{i=0}^{\lfloor\frac{n-1}{2}\rfloor}\frac{1}{(2i+1)!}\frac{K^{2i+1}}{N^{2i+1}}\ln^{2i+1}\left(\frac{K+r}{K+2i}\right)\\
        &\quad + \sum_{i=1}^{\lceil\frac{n-1}{2}\rceil}\frac{1}{(2i)!}\frac{K^{2i}}{N^{2i}}\ln^{2i}\left(\frac{K+r}{K-1}\right)\\
        & \overset{(b)}{\geq} 1 - \frac{K}{N}\ln\left(1 + \frac{n}{K}\right)\\
        &\quad + \sum_{i=1}^{\lfloor\frac{n-1}{2}\rfloor}\frac{1}{(2i)!}\frac{K^{2i}}{N^{2i}}\ln^{2i}\left(\frac{K+n}{K-1}\right)-\frac{1}{(2i+1)!}\frac{K^{2i+1}}{N^{2i+1}}\ln^{2i+1}\left(\frac{K+n}{K+2i}\right)\\
        & \overset{(c)}{\geq} 1 - \frac{K}{N}\ln\left(1 + \frac{n}{K}\right),
    \end{align*}
    where $(a)$ follows from Lemma \ref{le:multinomial_terms}, and $(b)$ follows from Lemma \ref{le:alt_terms}
\end{proof}
\subsection{Lemmas pertaining to Dirichlet Multinomial}
The following result upper bounds the expected number of unique classes drawn from a Dirichlet-multinomial distribution with $n$ draws and $\alpha = [K/N, \ldots, K/N]\in\Re^N$.

\begin{lemma}
    \label{le:num_unique}
    For all $n, K, N \in \mathbb{Z}_{++}$ s.t. $K \leq \sqrt{N}$ and $n \leq N$, if $\proxytheta' \sim {\rm DirMult}(n, \alpha)$ for $\alpha = \left[K/N, \ldots, K/N\right] \in \Re^{N}$, then
    $$\E\left[\sum_{i=1}^{N} \mathbbm{1}_{[\proxytheta' > 0]}\right] \leq K\ln\left(1 + \frac{n}{K}\right).$$
\end{lemma}
\begin{proof}
    \begin{align*}
        \E\left[\sum_{i=1}^{N} \mathbbm{1}_{[\proxytheta'_i > 0]}\right]
        & = N\cdot\Pr(\proxytheta'_i > 0)\\
        & \overset{(a)}{=} N\cdot\left(1 - \Pr(\proxytheta'_1=0, \proxytheta'_2+\cdots+\proxytheta'_N = n)\right)\\
        & = N\cdot\left(1 - \frac{\Gamma(K)\Gamma(n+1)}{\Gamma(n+K)}\cdot\frac{\Gamma\left(\frac{K}{N}\right)}{\Gamma\left(\frac{K}{N}\right)\Gamma\left(1\right)}\cdot\frac{\Gamma\left(n+\frac{K}{N}(N-1)\right)}{\Gamma\left(\frac{K}{N}(N-1)\right)\Gamma\left(n+1\right)}\right)\\
        & = N\cdot\left(1 -\frac{\Gamma(K)}{\Gamma(n+K)}\cdot\frac{\Gamma\left(n+K-\frac{K}{N}\right)}{\Gamma\left(K-\frac{K}{N}\right)}\right)\\
        & = N\cdot\left(1 -\frac{\prod_{i=0}^{n-1}\left(K-\frac{K}{N}+i\right)}{\prod_{i=0}^{n-1}K+i}\right)\\
        & = N\cdot\left(1 - \prod_{i=0}^{n-1}\left(1 - \frac{\frac{K}{N}}{K+i}\right)\right)\\
        & \overset{(b)}{\leq} N\cdot\left(1 - \left(1 - \frac{K}{N}\ln\left(1 + \frac{n}{K}\right)\right)\right)\\
        & = K\ln\left(1 + \frac{n}{K}\right).
    \end{align*}
    where $(a)$ follows from the aggregation property of the Dirichlet-multinomial distribution and $(b)$ follows from Lemma \ref{le:multi_lb}.
\end{proof}
Note that this upper bound is \emph{independent} of $N$. In the next section, we will apply the dominated convegence theorem to bound the number of unique basis functions drawn by our learning model.

\subsection{Extensions to Dirichlet Processes}
Recall that our learning model $\tilde{F}_{n,t}$ is identified by approximate output layer weights 
$$\proxytheta_w = \frac{\sign(\theta_w)\cdot\sum_{i=1}^{n}\mathbbm{1}_{w_i=w}}{n},$$ where $w_i \overset{iid}{\sim} \bar{\theta}$ and approximate basis function weights $\bar{\mathcal{W}} = \{\tilde{w}: w\in \tilde{\mathcal{W}}\}$ where $\tilde{w} = w+Z_w$, $\tilde{\mathcal{W}} = \{w \in \mathcal{W}: |\proxytheta| > 0\}$ and $Z_w \overset{iid}{\sim} \normal(0, \delta^2I_d)$.

In this section, we will translate the results from the previous section to the case in which the data generating process is the Dirichlet process that we study in the main text.

\begin{lemma}\label{le:num_unique_inf}
    For all $n, K \in \mathbb{Z}_{++}$,
    $$\E\left[\sum_{w\in\mathcal{W}}\mathbbm{1}_{[|\proxytheta_w| > 0]}\right] \leq K\ln\left(1+\frac{n}{K}\right).$$
\end{lemma}
\begin{proof}
    \begin{align*}
        \E\left[\sum_{w\in\mathcal{W}}\mathbbm{1}_{[|\proxytheta_w| > 0]}\right]
        & = \E\left[\E\left[\sum_{w\in\mathcal{W}}\mathbbm{1}_{[|\proxytheta_w|>0 ]}\bigg|\mathcal{W}\right]\right]\\
        & \overset{(a)}{=} \E\left[\E\left[\lim_{N\rightarrow\infty}\sum_{w\in\tilde{\mathcal{W}}}\mathbbm{1}_{[|\proxytheta_w'| > 0]}\bigg|\mathcal{W}\right]\right]\\
        & \overset{(b)}{=}
        \lim_{N\rightarrow\infty}\E\left[\E\left[\sum_{w\in\tilde{\mathcal{W}}}\mathbbm{1}_{[|\proxytheta_{w}'| > 0]}\bigg|\mathcal{W}\right]\right]\\
        & \overset{(c)}{\leq} \lim_{N\rightarrow\infty}K\ln\left(1+\frac{n}{K}\right)\\
        & = K\ln\left(1+\frac{n}{K}\right),
    \end{align*}
    where in $(a)$, $\mathcal{W}_N$ is a subset of the first $N$ elements of $\mathcal{W}$ and $\tilde{X}_N$ is ${\rm DirMult}(n, \alpha_N)$ where $\alpha_N = [K/N, \ldots, K/N] \in\Re^N$ and $\mathcal{W}_N$ is the set of classes, $(b)$ follows from the dominated convergence theorem since $|\sum_{w\in\tilde{\mathcal{W}}}\mathbbm{1}_{[\proxytheta_w' > 0]}|\leq n$, and $(c)$ follows from Lemma \ref{le:num_unique}.
\end{proof}

\begin{theorem}{\bf(entropy upper bound)}
    \label{th:rate_ub}
    For all $d, n, K \in \mathbb{Z}_{++}$ and $\delta > 0$,
    $$\H(\finfunc) \leq K\ln\left(1 + \frac{n}{K}\right)\cdot\left(\ln(2n) + d\ln\left(1+\frac{1}{d\delta^2}\right)\right).$$
\end{theorem}
\begin{proof}
    \begin{align*}
        \H(\finfunc)
        & \overset{(a)}{\leq} \E\left[\sum_{w\in\mathcal{W}} \mathbbm{1}_{|\proxytheta_w|>0}\cdot\left(\ln(2n) + \frac{d}{2}\ln\left(\frac{3}{\delta}\right)\right)\right]\\
        & \overset{(b)}{\leq} K\ln\left(1+\frac{n}{K}\right)\left(\ln(2n) + d\ln\left(\frac{3}{\delta}\right)\right)
    \end{align*}
    where $(a)$ follows from the fact that the output weight can take on at most $2n$ different values $(-\frac{n\sqrt{K+1}}{n}, -\frac{(n-1)\sqrt{K+1}}{n},\ldots,-\frac{\sqrt{K+1}}{n}, \frac{\sqrt{K+1}}{n},\ldots \frac{n\sqrt{K+1}}{n})$ and the fact that the $\delta$-covering number for $\sphere \leq (3/\delta)^d$ and $(b)$ follows from Lemma \ref{le:num_unique_inf}.
\end{proof}

\estErrorUb*
\begin{proof}
    The result follows from Lemma \ref{le:est_error_bound} and Theorem \ref{th:rate_ub}.
\end{proof}

\section{Characterization of the Efficient Frontier}
In this section, we provide a proof for Theorem \ref{th:eff_frontier}.
\label{apdx:eff_frontier}
\effFront*
\begin{proof}
    \begin{align*}
        n^*
        & = \argmin_{n\in\left[\frac{C}{d}\right]}\frac{3K}{n}+\frac{K\ln\left(1+\frac{n}{K}\right)\cdot\ln\left(2n\right)}{t} + \frac{dK\ln\left(1+\frac{n}{K}\right)\left(1 + \frac{1}{2}\ln\left(36t\right)\right)}{t};\ \text{ s.t. } n\cdot d\cdot t \leq C\\
        & = \argmin_{n\in\left[\frac{C}{d}\right]}\frac{3}{n}+\frac{\ln\left(1+\frac{n}{K}\right)\cdot\ln\left(2n\right)}{t} + \frac{d\ln\left(1+\frac{n}{K}\right)\left(1 + \frac{1}{2}\ln\left(36t\right)\right)}{t};\ \text{ s.t. } n\cdot d\cdot t \leq C\\
        & = \argmin_{n\in\left[\frac{C}{d}\right]}\frac{3}{n}+\frac{nd\ln\left(1+\frac{n}{K}\right)\cdot\ln\left(2n\right)}{C} + \frac{nd^2\ln\left(1+\frac{n}{K}\right)\left(1 + \frac{1}{2}\ln\left(\frac{36C}{nd}\right)\right)}{C}\\
        & \overset{(a)}{=} n \text{ s.t. } \frac{3}{n^2} = \frac{d\left(\ln\left(1+\frac{n}{K}\right) + \ln\left(1+\frac{n}{K}\right)ln(2n) + \frac{n}{K+n}\ln\left(2n\right)\right)}{C}\\
        &\qquad +\frac{d^2\ln\left(1+\frac{n}{K}\right)\ln\left(\frac{36C}{nd}\right)}{2C}+ \frac{n}{C(n+K)} + \frac{nd^2\ln\left(\frac{36C}{nd}\right)}{2C\left(n+K\right)}\\
        &\overset{(b)}{\leq} n \text{ s.t. } C = \frac{n^2d^2\ln\left(1+\frac{n}{K}\right)\ln\left(\frac{36C}{nd}\right)}{6}\\
        & = n \text{ s.t. } C - \frac{n^2d^2\ln(1+\frac{n}{K})}{6}\ln(C) = \frac{-n^2d^2\ln\left(1+\frac{n}{K}\right)\ln\left(\frac{nd}{36}\right)}{6}\\
        & \overset{(c)}{=} n \text{ s.t. } C = - \frac{n^2d^2\ln(1+\frac{n}{K})}{6}W_{-1}\left(-\frac{e^{\ln\left(\frac{nd}{36}\right)}}{ \frac{n^2d^2\ln(1+\frac{n}{K})}{6}}\right)\\
        & = n \text{ s.t. } C = - \frac{n^2d^2\ln(1+\frac{n}{K})}{6}W_{-1}\left(-\frac{nd}{6n^2d^2\ln(1+\frac{n}{K})}\right)\\
        & \overset{(d)}{\leq} n \text{ s.t. } C = - \frac{n^2d^2\ln(1+\frac{n}{K})}{6}W_{-1}\left(-\frac{1}{6nd\ln(1+\frac{n}{K})}\right)\\
        & \overset{(e)}{\leq} n \text{ s.t. } C =  \frac{n^2d^2\ln(1+\frac{n}{K})}{6}\ln\left(6nd\ln\left(1+\frac{n}{K}\right)\right)\\
        & \overset{(f)}{\leq} n \text{ s.t. } C = \frac{1}{6}n^2d^2\ln\left(1+\frac{n}{K}\right)^2\cdot \ln(\ln(d))\\
        & = n \text{ s.t. } \frac{6C}{d^2\ln(\ln(d))} = n^2\ln\left(1+\frac{n}{K}\right)^2\\
        & = n \text{ s.t. } \sqrt{\frac{6C}{d^2\ln(\ln(d))}} = n\ln\left(1+\frac{n}{K}\right)\\
        & \overset{(g)}{\leq} \min\left\{ \frac{\sqrt{\frac{6C}{d^2\ln(1+d)}}}{W_0\left(\sqrt{\frac{6C}{d^2K^2\ln(1+d)}}\right)}, \sqrt{\frac{6C}{d^2\ln(\ln(d))}} + \sqrt{\frac{6C}{d^2\ln(\ln(d))} + 4\sqrt{\frac{6CK^2}{d^2\ln(\ln(d))}}}\right\},
    \end{align*}
    where $(a)$ follows from $1$st order optimality conditions, $(b)$ follows from the fact that the rhs is an increasing function in $n$, so the root of a lower bound is an upper bound for the true root, $(c)$ follows from the fact that for all $a\in\Re$, the inverse function of $x - a\ln(x)$ is $-aW_{-1}\left(-\frac{e^{-\frac{x}{a}}}{a}\right)$ where $W_{-1}$ is the $-1$ branch of the Lambert W function, $(d)$ follows from the fact that the rhs is an increasing function in $n$ and so the root of a lower bound is an upper bound for the true root, $(e)$ follows from the fact that $W_{-1}(x) = \ln(-x)-\ln(-\ln(-x))+o(1)$ and monotonicity, $(f)$ follows from monotonicity and the fact that $\ln(6nd\ln(1+\frac{n}{K})) \geq \ln\left(1+\frac{x}{K}\right)\ln(\ln(d))$ when $\ln(6nd\ln(1+\frac{n}{K})) \geq 1$, and $(g)$ follows from the fact that for small $n/K$, $\ln(1+n/K)$ is well approximated by $\frac{n/K}{1+n/K}$ and the inverse of $n^2/(n+K)$ is $n + \sqrt{n^2 + 4Kn}$, and for large $n/K$, $\ln(1+n/K)$ is well approximated by $\ln(n/K)$ and the inverse of $nln(n/K)$ is $\frac{n}{W_0(x/K)}$. The result follows.
\end{proof}

\subsection{Optimal Quantization}
We begin with the following result which is an immediate consequence of Theorems \ref{th:mis_error_ub} and \ref{th:est_error_ub}.
\ceUB*
Corollary \ref{cor:ce_ub} still has a dependence on the free parameter $\delta^2$ i.e. the noise in our approximation of the basis functions. One way to eliminate the $\delta^2$ dependence is select the value of $\delta^2$ which minimizes the upper bound in Corollary \ref{cor:ce_ub}. The following result computes an upper bound on the expression that results after we perform this minimization.

\begin{restatable}{theorem}{optDelta}{\bf(optimal quantization)}\label{th:opt_noise}
    For all $n, d, K \in \mathbb{Z}_{++}$ and $\epsilon > 0$, if $n, K \geq 2$, then
    $$ \inf_{\delta^2 > 0} 2d\delta^2 + \frac{dK\ln\left(1+\frac{n}{K}\right)\ln\left(\frac{3}{\delta}\right)}{t} \leq \frac{dK\ln\left(1+\frac{n}{K}\right)\left(1+\ln\left(36t\right)\right)}{2t}.
    $$
\end{restatable}
\begin{proof}
    \begin{align*}
        \argmin_{\delta^2 > 0} 2d\delta^2 + \frac{dK\ln\left(1+\frac{n}{K}\right)\ln\left(\frac{3}{\delta}\right)}{t}
        & = \argmin_{\delta^2 > 0} 2\delta^2 + \frac{K\ln\left(1+\frac{n}{K}\right)\ln\left(\frac{9}{\delta^2}\right)}{2t}\\
        & \overset{(a)}{=} \delta^2 \text{ s.t. } 0 = 2-\frac{K\ln\left(1+\frac{n}{K}\right)}{2t}\cdot\frac{1}{\delta^2}\\
        & = \delta^2 \text{ s.t. } 4 = \frac{K\ln\left(1+\frac{n}{K}\right)}{t}\cdot\frac{1}{\delta^2}\\
        & = \frac{K}{4t}\ln\left(1 + \frac{n}{K}\right)\\
    \end{align*}
    where $(a)$ follows from first order optimality conditions.
    
    Plugging in this value of $\delta^2$ gives:
    \begin{align*}
        \inf_{\delta^2 > 0} 2d\delta^2 + \frac{dK\ln\left(1+\frac{n}{K}\right)\ln\left(\frac{9}{\delta^2}\right)}{2t}
        & \leq \frac{dK\ln\left(1+\frac{n}{K}\right)\left(1+\ln\left(\frac{36t}{K\ln(1+\frac{n}{K})}\right)\right)}{2t}\\
        & \overset{(a)}{\leq} \frac{dK\ln\left(1+\frac{n}{K}\right)\left(1+\ln\left(36t\right)\right)}{2t},
    \end{align*}
    where $(a)$ follows from the fact that $\frac{36t}{K\ln\left(1+\frac{n}{K}\right)} \leq 36t$ for $n, K \geq 2$. The result follows.
\end{proof}

\section{Experimental Details}
\label{apdx:experiment}
\subsection{Data Generating Process}
Figure 1 was generated by setting $d = 16$ and $K=16$. We implemented the data generating process in the following way. We let
$$F(X) = \sqrt{K}\sum_{i=1}^{N} \theta_i \cdot \omega_i \cdot \relu(A_i^\top X).$$
Instead of sampling directly from a Dirichlet process with base distribution ${\rm Uniform}(\sphere)$ and scale parameter $K$, we approximate this via sampling linear layer weights $\theta$ from an $N$ dimensional dirichlet random variable with parameter $(K/N, \ldots, K/N)$ and then sampling $\omega_i \overset{iid}{\sim} {\rm Rademacher}$. The the basis functions are sampled $A_i \overset{iid}{\sim} {\rm Uniform}(\sphere)$. It is well known that as $N \rightarrow \infty$ this process converges to the Dirichet process that we study theoretically. For our experiments, we choose $N = 2^{12}$.

We let $X_t \sim \normal(0, I_d)$ and generate $Y_{t+1} = 1/(1 + \exp\{-F(X_t)\})$. Note that in our experiments, we do not actually sample $Y_{t+1}$ as a binary random variable with probability $1/(1 + \exp\{-F(X_t)\})$ but instead provide the probability itself. This is done primarily to reduce variance in the experiments as the two data generating processes result in optimizing the same objective but providing the probabilities reduces the variance of the empirical risk.

\subsection{Efficient Frontier Sweep}
We construct a geometric sequence of tractable compute budgets for which we empirically compute the efficient frontier. For these experiments we used the sequence $\{2^{18}, 2^{19}, \ldots, 2^{26}\}$. For each value of $C$ in this sequence, we construct a set of pairs $(n, t)$ for which $nt = C/d$. For each pair $(n, t)$, we generated a training dataset of size $t$ and a learning network of input dimension $d$ and width $n$. For each dataset, we performed $10$ training runs of Adam optimizer and record the final resulting test error. We perform these $10$ runs to smooth out anomalous cases as neural networks can be finicky. We detail the training procedure in the following subsection.

\subsection{Training}

Given $t$ training samples and a training network of width $n$, we endow each training run with a validation dataset of fixed size $10000$. Note that as the compute budget grows, this validation set is quickly dwarfed by the size of the training dataset so we ignore its impacts the flop count. We perform $500$ epochs of minibatch stochastic gradient descent with adam optimizer and batchsize $1024$. We implemented an early stopping procedure in which we return the model within the 500 epochs with the lowest validation error. This returned model is separately evaluated on a test set of size $100000$ and the out of sample cross-entropy error is recorded.
\newpage
\begin{figure}
    \centering
    \includegraphics[scale=0.4]{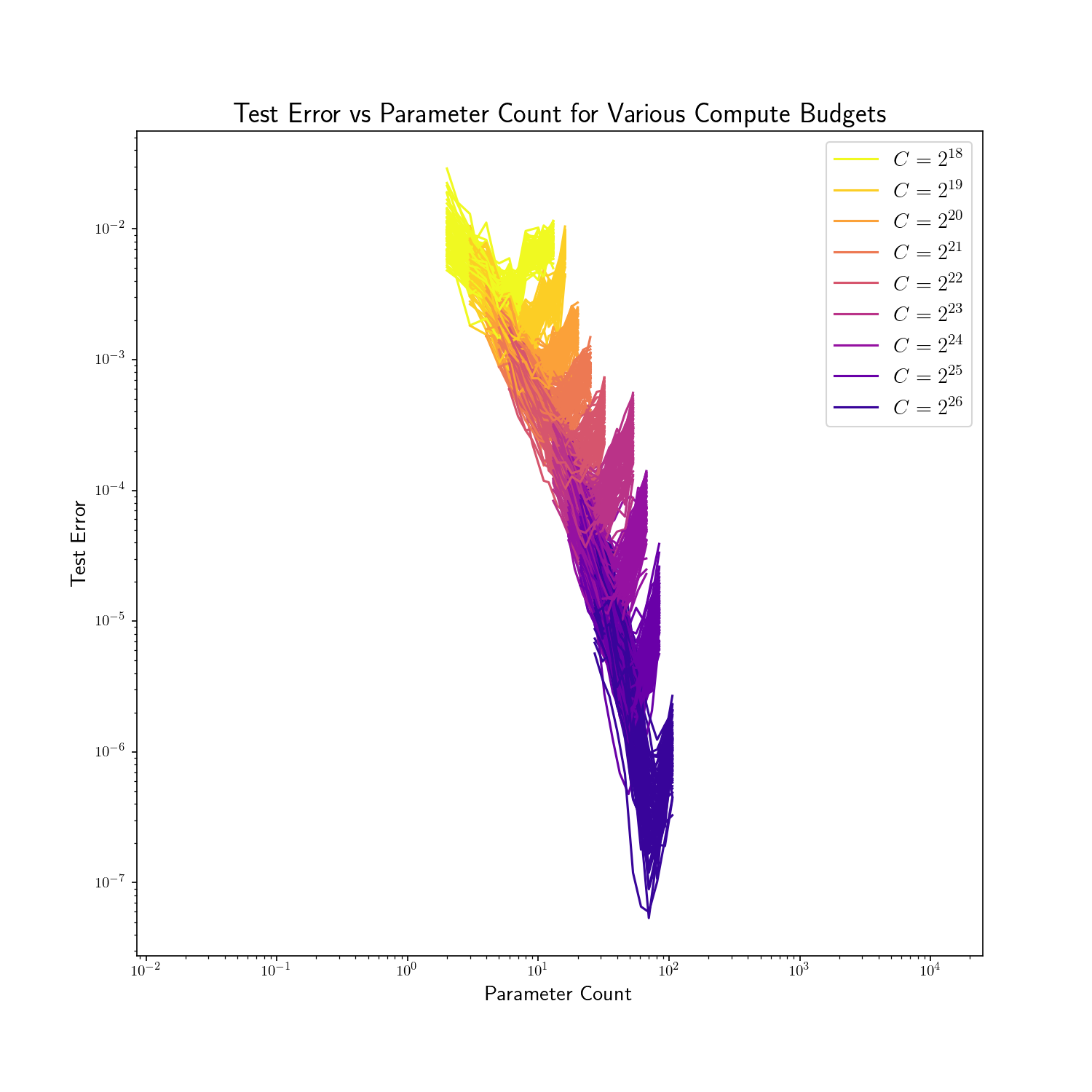}
    \caption{We plot each of the $78\times 9 \times 10$ aforementioned trials.}
    \label{fig:frontier_sweeps}
\end{figure}

\subsection{Bootstrapping and Extrapolation}
For $K=16, d=16$, we performed the above procedure for $78$ different realizations of $F$. For each realization, we perform an efficient frontier sweep over $9$ different compute values ranging from $2^{18}$ to $2^{26}$. For each compute value, we perform a sweep over pairs $(n, t)$ which satisfy $n t = C/d$ and for each, generate a dataset of size $t$ from $F$. With this training dataset, we initialize $10$ learning models with input dimension $d$ and width $n$ and perform the training procedure outlined above. The results of this procedure are displayed in Figure \ref{fig:frontier_sweeps}. We average the test errors of the model returned by each of these $10$ runs. The result is that for each $(n, t)$ pair, we have $78$ values for the average test error based on different realizations of the data generating process. We perform a bootstrapping procedure to generate the confidence intervals in Figure \ref{fig:efficient-frontier}.

In our bootstrapping, we subsample $10$ realizations with replacement from the $79$ and perform 4 linear fits in logarithmic scale. Recall that we evaluated $9$ different compute values so the first linear fit performs least squares regression where $X$ is the subsampled log optimal training set sizes for compute values $\{2^{18}, \ldots, 2^{21}\}$ and $Y$ is the corresponding log optimal parameter counts. The second linear fit is for compute values $\{2^{20}, \ldots, 2^{23}\}$, the third for compute values $\{2^{22},\ldots, 2^{25}\}$, and the fourth for $\{2^{23},\ldots, 2^{26}\}$. To generate Figure \ref{fig:efficient-frontier}, we performed this bootstrapping process $10000$ times, each generating $4$ linear fits to the log optimal dataset size and parameter count.

To extrapolate to compute values higher than $2^{26}$, we compute a mean of just the $4$th linear fit of each of the $10000$ subsamples. The confidence intervals provided in Figure \ref{fig:efficient-frontier} are $1$-standard deviation above and below this mean value computed from the $10000$ subsamples. Likewise, for extrapolation to compute values lower than $2^{18}$, we perform the same procedure but with the $1$st linear fit. 

For interpolation within the compute values we evaluated, we performed a moving average where values at the extremes $(2^{18} \text{ and }2^{26})$ were evaluated just via averages of linear fits $1$ and $4$ respectively, but values in the middle were averages of linear fits in which they were a part of i.e. the value of $2^{20}$ would be an average of linear fits $1$ and $2$ since both of these linear fits used data from compute value $2^{20}$.

For reference, generating the results to produce Figure \ref{fig:efficient-frontier} took over $1$ week when run in parallel across $2$ machines each with $32$ cpu cores and $8$gb gpus. In Figure \ref{fig:other_eff_frontier}, we provide results for $(d=16, K=64)$ and $(d=64, K=16)$ which perform the same procedure as above but only over $20$ realizations of $F$. They are marginally less convincing than the result of Figure \ref{fig:efficient-frontier} but this could be due to noise in the data from having fewer runs. It would be interesting to perform much more comprehensive and higher-scale experiments in the future, but for the purposes of this work, we found this experimental evidence compelling enough.

\begin{figure}
    \centering
    \includegraphics[scale=0.53]{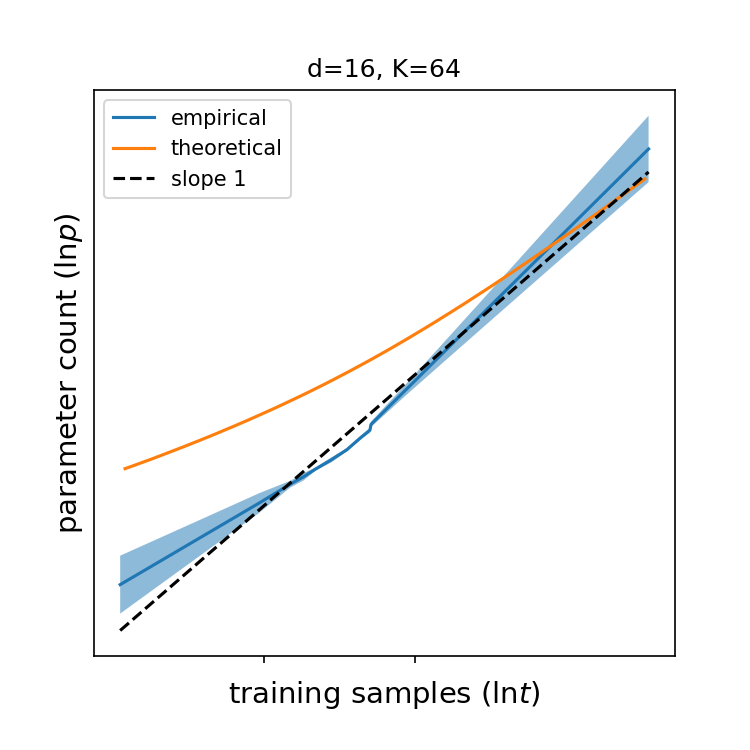}
    \includegraphics[scale=0.53]{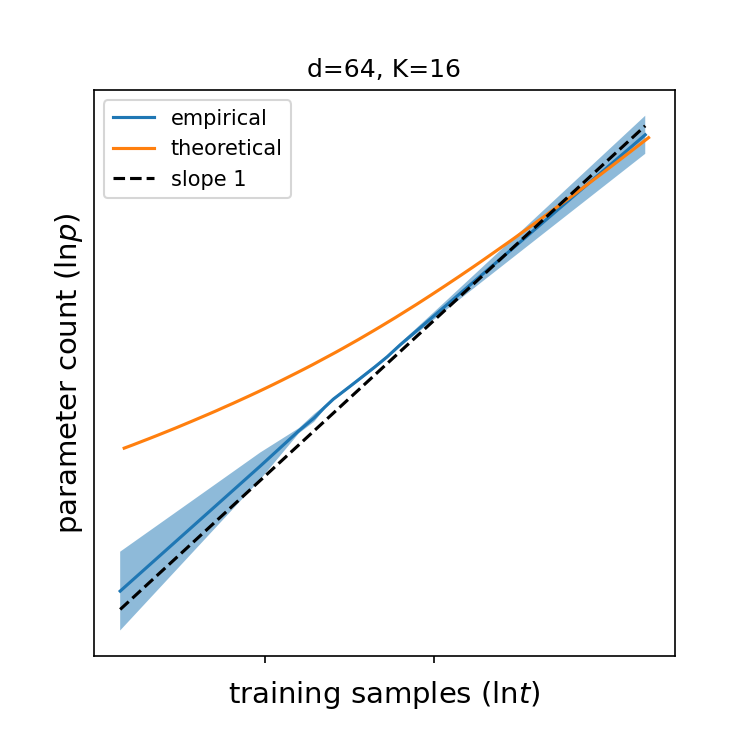}
    \caption{We produce the same plot as in Figure \ref{fig:efficient-frontier} but for $(d=16, K=64)$ and $(d=64, K=16)$. Note that we only perform averaging over $20$ realizations as opposed to $78$ and for our bootstrapping procedure we subsample $5$ realizations as opposed to $10$.}
    \label{fig:other_eff_frontier}
\end{figure}

\end{document}